\documentclass[letterpaper, 10 pt, conference]{ieeeconf}  % Comment this line out if you need a4paper

\IEEEoverridecommandlockouts                              % This command is only needed if 
\usepackage{graphicx} % for PDF, bitmapped graphics files
\usepackage{epsfig} % for postscript graphics files
\usepackage{amsmath} % assumes amsmath package installed
\usepackage{amssymb}  % assumes amsmath package installed
\usepackage[T1]{fontenc}
\usepackage{amsfonts} 
\usepackage{import}
\usepackage{xcolor}
\usepackage{float}
\usepackage{cite}
\usepackage[caption=false,font=footnotesize]{subfig}
\usepackage{booktabs}
\usepackage{tabularx}

\newtheorem{theorem}{Theorem}
\newtheorem{lemma}{Lemma}
\newtheorem{problem}{Problem}
\newtheorem{remark}{Remark}
\newtheorem{example}{Example}

\DeclareMathOperator*{\argmin}{argmin}

\title{\LARGE \bf
Least Restrictive Hyperplane Control Barrier Functions 
}

\author{Mattias Trende and Petter Ögren% <-this % stops a space
\thanks{The authors are with the Robotics, Perception and Learning Lab., School of Electrical Engineering and Computer Science, Royal Institute of Technology (KTH), SE-100 44 Stockholm, Sweden, {\tt\small \{mtrende, petter\}@kth.se}}%
}

\begin{document}

\maketitle
\thispagestyle{empty}
\pagestyle{empty}

%%%%%%%%%%%%%%%%%%%%%%%%%%%%%%%%%%%%%%%%%%%%%%%%%%%%%%%%%%%%%%%%%%%%%%%%%%%%%%%%
\begin{abstract}

Control Barrier Functions (CBFs) can provide provable safety guarantees for dynamic systems. 
However, finding a valid CBF for a system of interest is often non-trivial, especially for systems with low computational resources, higher-order dynamics, and moving close to obstacles of complex shape. 
A common solution to this problem is to use a purely distance-based CBF. In this paper, we study Hyperplane CBFs (H-CBFs), where a hyperplane separates the agent from the obstacle. 
First, we note that the common distance-based CBF is a special case of an H-CBF where the hyperplane is a supporting hyperplane of the obstacle that is orthogonal to a line between the agent and the closest point of the obstacle. 
We then show that a less conservative CBF can be found by optimising over the orientation of the supporting hyperplane, in order to find the Least Restrictive Hyperplane CBF. 
This enables us to maintain the safety guarantees while allowing controls that are closer to the desired ones, especially when moving fast and passing close to obstacles. 
We illustrate the approach on a double integrator dynamical system with acceleration constraints, moving through a group of arbitrarily shaped static and moving obstacles, and show that the proposed approach reduces the average gap between desired and safe controls by an order of magnitude.

% Control Barrier Functions (CBFs) can provide provable safety guarantees for dynamic systems. However, finding a valid CBF for a system of interest is often non-trivial, especially for systems with low computational resources, higher-order dynamics, and moving close to obstacles of complex shape. A common solution to this problem is to use a purely distance-based CBF. In this paper, we study Hyperplane CBFs (H-CBFs), where a hyperplane separates the agent from the obstacle. First, we note that the common distance-based CBF is a special case of an H-CBF where the hyperplane is a supporting hyperplane of the obstacle that is orthogonal to a line between the agent and the closest point of the obstacle. We then show that a less conservative CBF can be found by optimising over the orientation of the supporting hyperplane, in order to find the Least Restrictive Hyperplane CBF. This enables us to maintain the safety guarantees while allowing controls that are closer to the desired ones, especially when moving fast and passing close to obstacles. We illustrate the approach on a double integrator dynamical system with acceleration constraints, moving through a group of arbitrarily shaped static and moving obstacles, and show that the proposed approach reduces the average gap between desired and safe controls by an order of magnitude.

\end{abstract}

%%%%%%%%%%%%%%%%%%%%%%%%%%%%%%%%%%%%%%%%%%%%%%%%%%%%%%%%%%%%%%%%%%%%%%%%%%%%%%%%
\section{Introduction}

Safety is crucial for many robotic systems, whether they are fully autonomous or include humans-in-the-loop. Over the last two decades, Control Barrier Functions (CBFs) have gained popularity as a means to ensure safety in robotics \cite{prajna_framework_2007, ames_control_2014, ames_control_2019}. The role of a CBF is typically to serve as a safety filter. Given the state, the CBF can provide a set of safe control actions, and given a desired control action and a CBF, one can determine the safe control action that is closest to the desired one.

The difficult part of using CBFs to ensure safety is often to find a valid CBF \cite{hsuSafetyFilterUnified2024}. For low-dimensional systems with simple obstacle geometries, it is often possible to derive a CBF analytically \cite{chen_backup_2021}, but for more complex geometries and higher-dimensional systems, a closed-form solution might be quite conservative, while numerical solutions can require extensive computations \cite{hsuSafetyFilterUnified2024}. 
A common solution to this problem is to utilise some type of approximate CBF. 
% While there are machine learning approaches to the problem \cite{yu_sequential_2023}, using a black box model to guarantee safety can create additional challenges. 

A very popular, but conservative, approach is the distance-based CBF, used in e.g., \cite{borrmann_control_2015, wang_safety_2017, molnar_collision_2025, funada_collision_2025, thyri_reactive_2020}.
The key idea is to keep track of the distance to the closest part of the obstacle, and make sure it stays positive.
It turns out that the distance-based CBF is equivalent to positioning a supporting hyperplane at the closest point of the obstacle, orthogonal to the line between the obstacle and the agent, and applying a CBF that guarantees that the agent stays on the right side of this hyperplane. We call this method the Minimum Distance Hyperplane-CBF (MDH-CBF), and the more general approach of using any supporting plane a Hyperplane-CBF (H-CBF).

By using an H-CBF, the problem dimensionality is essentially reduced by projecting the system's movement orthogonally to the hyperplane \cite{borrmann_control_2015}. This enables the use of analytically derived CBFs for higher-dimensional systems to guarantee safety.
By properly choosing the hyperplane, one can handle many obstacle geometries, such as circles, ellipses and polygons \cite{borrmann_control_2015, funada_collision_2025, liu_safety-critical_2025}.
  
The downside of using hyperplanes to describe obstacles is that it can be quite conservative, significantly restricting aggressive manoeuvring (high speeds and high accelerations) of the vehicle close to the obstacle.
However, in this paper, we note that conservativeness is only a problem if it prevents the agent from doing what it wants. 
Thus, we propose to find \emph{the least restrictive H-CBF} as follows.

We start by noting that the standard CBF approach \cite{ames_control_2019} involves \emph{first} choosing a CBF (without any knowledge of $\mathbf{u}_{\text{des}}$), and \emph{then}, at each time step, given a desired (possibly non-safe) control
$\mathbf{u}_{\text{des}}$,
finding a safe control $\mathbf{u}$ that minimises $||\mathbf{u}- \mathbf{u}_{\text{des}} ||$.

\begin{figure}[t!]
    \centering
    \def\svgwidth{0.35\textwidth}%\footnotesize
    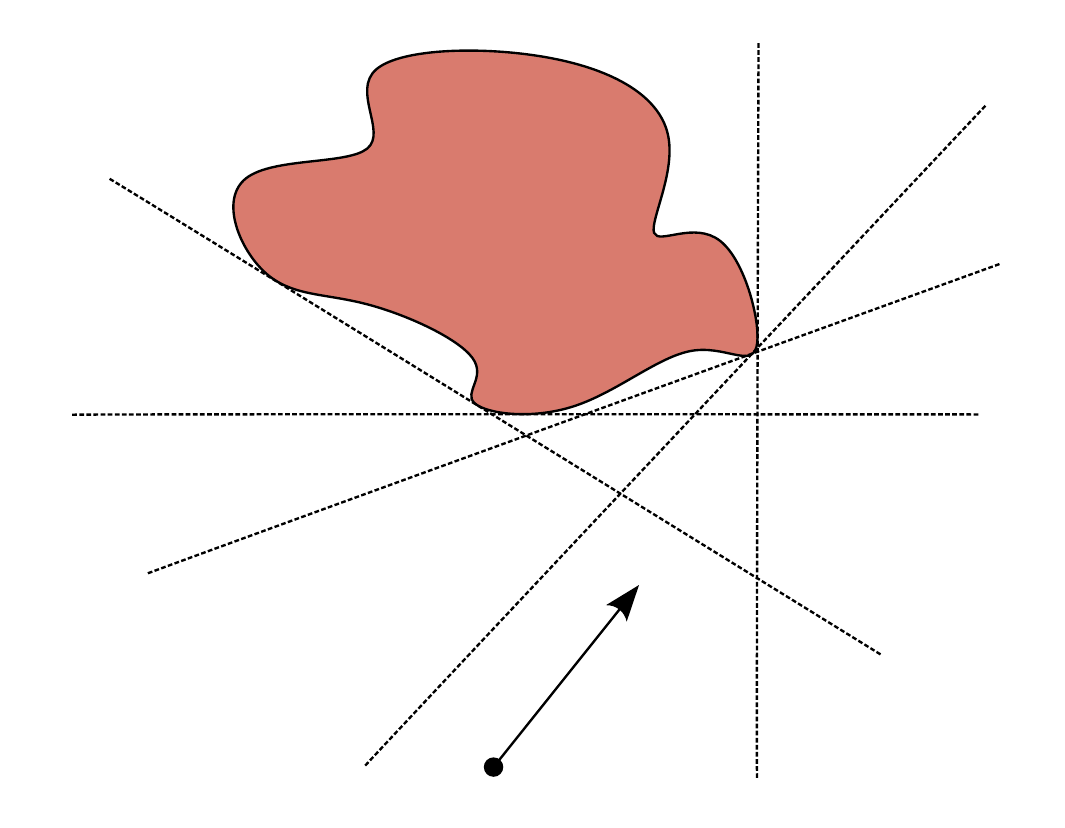
    \caption{An agent $A$ with velocity $\mathbf{v}_A$ in front of an obstacle $O$. Five of the infinitely many supporting hyperplanes, parameterised by some $\theta$, are shown ($P_{1}-P_5$). 
    A common choice used in \cite{borrmann_control_2015, wang_safety_2017, molnar_collision_2025, funada_collision_2025, thyri_reactive_2020} is $P_2$ (denoted MDH-CBF) going through the closest point on the obstacle.
    Note, however, that if the goal of the agent is at $G_1$, an H-CBF based on $P_4$ would make more sense (allowing the agent to pass the obstacle at full speed), whereas if the goal is at $G_2$, an H-CBF based on $P_1$  might be better (allowing a sharp left turn). The key idea in this paper is that, given the desired action of the agent $\mathbf{u}_{\text{des}}$, we optimise over a family of CBFs using $\theta$ to find the \emph{least restrictive} H-CBF, allowing a safe control $\mathbf{u}$ that is as close as possible to the desired $\mathbf{u}_{\text{des}}$.}
    \label{fig:page1}
\end{figure}

Then, in contrast to the above, the key part of the proposed approach is that instead of picking a single CBF, we pick a family of CBFs, parameterised by some $\theta$, and then, in each time step, pick the individual CBF from the family that enables the smallest deviation from $\mathbf{u}_{\text{des}}$, i.e., the combined choice of $(\theta, \mathbf{u})$ that makes $\mathbf{u}$ safe and minimises $||\mathbf{u}- \mathbf{u}_{\text{des}} ||$. We call this approach the \emph{Least Restrictive Hyperplane} CBF (LRH-CBF).
A conceptual illustration of this approach is provided in Figure~\ref{fig:page1}. 

Thus, the main contribution of this paper is a generalisation of the distance-based CBF approach from \cite{borrmann_control_2015, wang_safety_2017, molnar_collision_2025, funada_collision_2025, thyri_reactive_2020}, in which, instead of choosing an MDH-CBF, we choose the LRH-CBF with respect to the desired control action, thereby minimising the gap between the chosen safe control and the desired control, making the approach less conservative.
In our experiments, this gap is on average a factor of 12 smaller for LRH-CBF, and a set of test problems is completed faster, with less control intervention.

% In the remainder of this paper,
% Section \ref{section: Related work} looks further into the connections to previous work, while
% Section \ref{section: Background} summarises the theory of CBFs as well as the double integrator model with a corresponding H-CBF.  Then, Section \ref{section: Method} describes the proposed approach. Finally, the simulation results are presented in Section \ref{section: Experiments}, and the paper is concluded in Section \ref{section: Conclusion}.

\section{Related Work}
\label{section: Related work}

There has been extensive work on CBFs for different representations of obstacles.
The four most common ways are circles \cite{borrmann_control_2015, wang_safety_2017}, ellipses \cite{funada_collision_2025, tooranjipour_lidar-based_2025},  polygons \cite{thirugnanam_safety-critical_2022, wu_optimization-free_2025, wei_collision_2025, notomista_reactive_2025, dai_safe_2023},
and hyperplanes \cite{borrmann_control_2015,wang_safety_2017, molnar_collision_2025, funada_collision_2025,thyri_reactive_2020, huang_dynamic_2025,roncero_multi-agent_2025-1,liu_safety-critical_2025}.
In the last category, we also include the common CBF approach that tracks the distance from the agent to the closest point on an obstacle and ensures this distance remains positive \cite{borrmann_control_2015,ghaffariSafetyVerificationUsing2018,roncero_multi-agent_2025-1}. Note that even though this approach does not explicitly mention hyperplanes,  it is equivalent to an H-CBF touching the closest point of the obstacle, with a normal vector pointing towards the agent (i.e., an MDH-CBF).
% Throughout the paper, we refer to these H-CBFs as \emph{Orthogonal} (OH-CBFs), as they are orthogonal to the line between the agent and the closest point on the obstacle.

There are other works suggesting different orientations of H-CBFs.
In \cite{funada_collision_2025}, ellipsoidal bodies are considered, and the orientation of the hyperplane is chosen to maximise the distance between the hyperplane and the other nearby ellipsoid.
In \cite{huang_dynamic_2025}, the orientation of the H-CBF is chosen based on velocity obstacles (VO). They consider a second-order dynamic system, and to avoid working with higher-order CBFs, they create the CBFs in velocity space, aligned with the VOs.
In \cite{liu_safety-critical_2025}, H-CBFs are picked based on a safe convex polytope, generated from free space in an occupancy grid map.

% In this paper, the parameterisation of the family of H-CBFs connected to each obstacle was inspired by a description that was originally used for artificial potential fields \cite{rimon_construction_1991}, and then repurposed for CBFs \cite{notomista_safety_2022, notomista_reactive_2025}. However, the proposed approach in this paper does not perform any map transformations, and most importantly, the approaches in \cite{notomista_safety_2022, notomista_reactive_2025} do not minimize over the CBF candidates to find the least restrictive one in terms of $||\mathbf{u}- \mathbf{u}_{\text{des}} ||$.

% We also note that our use of H-CBFs makes the proposed approach suitable for combinations with several other methods in the literature, such as work on multi-agent systems
% \cite{wang_safety_2017, an_collisions-free_2022, tan_distributed_2022}, cooperative/neutral/adversarial obstacles \cite{borrmann_control_2015, martinez-baselga_avocado_2025} 
% and MPC \cite{tooranjipour_lidar-based_2025, abbas_obstacle_2017}.

Finally, we note that this paper is different from all the approaches mentioned above in the sense that the choice of the H-CBF is dependent on the desired control.
% Given a desired control $\mathbf{u}_{\text{des}}$, we jointly optimise over the CBF family $\theta$ and the safe control $\mathbf{u}$ to minimise
% $||\mathbf{u}- \mathbf{u}_{\text{des}} ||$.
% Thus, it finds the LRH-CBF in terms of enabling a safe control as close as possible to the desired control.
% None of the other papers considers the desired control when choosing the CBF, instead, they base it purely on the current state.

% We believe that there are advantages to using the desired control to optimise an H-CBF, as illustrated in Figure~\ref{fig:page1}.

\section{Background}
\label{section: Background}

In this section, we first summarise continuous and discrete CBF theory, then describe a double integrator model and a common H-CBF for it.

% \subsection{Convex Analysis}
% \label{sec_convex_analysis}

% \begin{definition}
% A supporting hyperplane to a set $S \subset \mathbb{R}^n$
% at a point $\mathbf{x}_0 \in \partial S$ (the boundary of $S$)
% is a hyperplane that:
% passes through the point $\mathbf{x}_0$ and leaves the entire set 
%  on one side of the hyperplane (or on the hyperplane itself).
% \end{definition}

% \begin{theorem}[Supporting Hyperplane in Every Direction]
% \label{th_every_direction}
%     If a set $S \subset \mathbb{R}^n$ is closed, convex, bounded, and has nonempty interior, then for every nonzero direction $\mathbf{n} \in \mathbb{R}^n$ there exists a point $\mathbf{p}^* \in S$ such that
%     \begin{equation}
%         \mathbf{n}^T \mathbf{p}^* = \max_{\mathbf{p} \in S} \mathbf{n}^T \mathbf{p}
%     \end{equation}
%     and the hyperplane $\{\mathbf{p}: \mathbf{n}^T \mathbf{p} = \mathbf{n}^T \mathbf{p}^*\}$ supports $S$ at $\mathbf{p}^*$.
% \end{theorem}
% \begin{proof}
%     See \cite{rockafellarConvexAnalysis1970}.
% \end{proof}

% \begin{theorem}[Separating Hyperplane]
% \label{th_sep_hyperplane}
% Let $C_1, C_2 \subset \mathbb{R}^n$ be two compact, convex, nonempty, disjoint sets. Then, there exists a separating hyperplane $\mathbf{a}^T \mathbf{p} = b$, such that
% \begin{align}
%     \mathbf{a}^T \mathbf{p} \leq b & \quad \forall \mathbf{p} \in C_1 \\
%     \mathbf{a}^T \mathbf{p} \geq b & \quad \forall \mathbf{p} \in C_2. 
% \end{align}
% \end{theorem}

% \begin{proof}
%     See \cite{rockafellarConvexAnalysis1970}.
% \end{proof}

\subsection{Control Barrier Functions}
Following the CBF definition in \cite{ames_control_2019}, consider a nonlinear affine control system in the form 
$\dot{\mathbf{x}} = \mathbf{f}(\mathbf{x}) + \mathbf{g}(\mathbf{x})\mathbf{u},$
where $\mathbf{f}$ and $\mathbf{g}$ are locally Lipschitz, $\mathbf{x} \in \mathbb{R}^n$ is the system state, and $\mathbf{u} \in \mathbb{R}^m$ is the system input. 
Then introduce the safe set $\mathcal{C} \subset \mathbb{R}^n$ given by $\mathcal{C} = \{\mathbf{x} \in \mathbb{R}^n \ | \ h(\mathbf{x} ) \geq 0\}$ for some continuously differentiable function $h:\mathbb{R}^n \to \mathbb{R}$.
For safety, $\mathcal{C} \cap O = \emptyset$ must hold, where $O$ is the obstacle set.
If $h(\mathbf{x)}$ also satisfies the inequality 
\begin{equation} \label{eq: h_lim}\sup_{\mathbf{u} \in \mathcal{U}}
    \dot{h}(\mathbf{x},\mathbf{u}) \geq -\alpha(h(\mathbf{x})),
\end{equation}
for all $\mathbf{x} \in \mathcal{C}$, then $h$ is a CBF. Here $\alpha : \mathbb{R} \to \mathbb{R}$ is an extended class $\mathcal{K}_\infty$ function. This defines $\alpha$ to be a strictly increasing function that satisfies $\alpha(0) = 0$ and $\lim_{x\to\pm\infty}\alpha(x) = \pm \infty$.

% Following the CBF definition in \cite{ames_control_2019}, consider a nonlinear affine control system in the form 
% $\dot{\mathbf{x}} = \mathbf{f}(\mathbf{x}) + \mathbf{g}(\mathbf{x})\mathbf{u},$
% where $\mathbf{f}$ and $\mathbf{g}$ are locally Lipschitz, $\mathbf{x} \in \mathbb{R}^n$ is the system state, and $\mathbf{u} \in \mathbb{R}^m$ is the system input. Then introduce the safe set $\mathcal{C}$, defined such that for every state $\mathbf{x}(t) \in \mathcal{C}$ there exists a permissible control $\mathbf{u} \in \mathcal{U}$ that ensures $\mathbf{x}(t + \epsilon) \in \mathcal{C}$ for all $\epsilon > 0$. Consider a continuously differentiable function $h:\mathbb{R}^n \to \mathbb{R}$. If $h$ is such that
% $h(\mathbf{x}) \geq 0 \ \iff \mathbf{x} \in \mathcal{C}$
% and satisfies the inequality 
% \begin{equation} \label{eq: h_lim}\sup_{\mathbf{u}}
%     \dot{h}(\mathbf{x},\mathbf{u}) \geq -\alpha(h(\mathbf{x})),
% \end{equation}
% then $h$ is a CBF. Here $\alpha : \mathbb{R} \to \mathbb{R}$ is an extended class $\mathcal{K}_\infty$ function. This defines $\alpha$ to be a strictly increasing function that satisfies $\alpha(0) = 0$ and $\lim_{x\to\pm\infty}\alpha(x) = \pm \infty$. 

Given a desired control $\mathbf{u}_{\text{des}}$, the search for a safe control $\mathbf{u}$ close to the desired one can be formalised in the following optimisation problem \cite{ames_control_2019}:

\begin{problem}[CBF-QP]
\label{prob_CBF_QP}
\begin{align}
    \mathbf{u}(\mathbf{x}) = \argmin_\mathbf{u} & \ (\mathbf{u}-\mathbf{u}_{\text{des}})^T \mathbf{Q} (\mathbf{u}-\mathbf{u}_{\text{des}}) \\
    \text{s.t. } 
    & \ 
    \dot{h}(\mathbf{x},\mathbf{u}) \geq -\alpha(h(\mathbf{x})) \label{eq: cont cons CBF}\\
    & \ \mathbf{u} \in \mathcal{U},
\end{align}
where $Q$ is a positive definite matrix.
\end{problem}

It can be shown that if the dynamics are as above, and $\mathcal{U} = \mathbb{R}^m$, the problem above is indeed a Quadratic Programming problem (QP) for which efficient solvers exist \cite{ames_control_2019}.

\subsection{Discrete Control Barrier Functions}

When implementing CBFs on discrete-time systems, such as systems with modern embedded electronics, the CBF needs to be discretised into a Discrete-CBF (D-CBF) \cite{agrawal_discrete_2017}. Instead of resulting in a QP for nonlinear affine systems, the problem becomes a Quadratically Constrained QP (QCQP), where \eqref{eq: cont cons CBF} is replaced with
\begin{equation} \label{eq: DCBF alt cnstr}
    h(\mathbf{x}_{k+1}) \geq (1 - \gamma) h(\mathbf{x}_k).
\end{equation}

% \begin{problem}[D-CBF-QCQP]
    % \begin{align}
    %     \mathbf{u}(\mathbf{x}) = \argmin_\mathbf{u} & \ (\mathbf{u}-\mathbf{u}_{\text{des}})^T \mathbf{Q} (\mathbf{u}-\mathbf{u}_{\text{des}}) \\
    %     \text{s.t. } 
    %     & \ 
    %     \Delta h(\mathbf{x}_{k+1},\mathbf{u}_k) \geq - \gamma h(\mathbf{x}_k) \label{eq: DCBF} \\
    %     & \ \gamma \in (0, 1]\\
    %     & \ \mathbf{u} \in \mathcal{U},
    % \end{align}
%     where $Q$ is a positive definite matrix, the difference is $\Delta h(\mathbf{x}_{k+1}) = h(\mathbf{x}_{k+1}) - h(\mathbf{x}_k)$, and the discrete dynamics follow $\mathbf{x}_{k+1} = \mathbf{f}(\mathbf{x}_{k}) + \mathbf{g}(\mathbf{x}_k) \mathbf{u}_k$. This means that the constraint in \eqref{eq: DCBF} can be rewritten as
%     \begin{equation} \label{eq: DCBF alt cnstr}
%         h(\mathbf{x}_{k+1}) \geq (1 - \gamma) h(\mathbf{x}_k).
%     \end{equation}
% \end{problem}

\subsection{Double Integrator Model}
\label{sec: DI model}

In this paper, we consider a double integrator dynamic model subject to acceleration constraints. Let the state be $\mathbf{x}=(\mathbf{p}, \mathbf{v})\in \mathbb{R}^{4}$ and
$\dot{\mathbf{p}}  = \mathbf{v},\  \dot{\mathbf{v}}  = \mathbf{u}$
% \begin{align}
%     \dot{\mathbf{p}}  &= \mathbf{v} \label{eq_double_integrator}\\
%     \dot{\mathbf{v}}  &= \mathbf{u} \nonumber
% \end{align}
with the acceleration constraint $||\mathbf{u}|| \leq u_{\max}$.
Then the system can be written in control-affine form
% \begin{equation}
%     \dot{\mathbf{x}} = 
%     \mathbf{f(x)} + \mathbf{g(x)} \mathbf{u} =
%     \mathbf{A} \mathbf{x} + \mathbf{B} \mathbf{u},
% \end{equation}
for the proper choice of $\mathbf{f}$ and $\mathbf{g}$.
% where $\mathbf{A}$ and $\mathbf{B}$ are given by,
% \begin{align}
%     \mathbf{A} = 
%         \begin{bmatrix}
%             0 & 0 & 1 & 0\\
%             0 & 0 & 0 & 1\\
%             0 & 0 & 0 & 0\\
%             0 & 0 & 0 & 0
%         \end{bmatrix}, \
%     \mathbf{B} = 
%         \begin{bmatrix}
%             0 & 0 \\
%             0 & 0 \\
%             1 & 0 \\
%             0 & 1 
%         \end{bmatrix}.
% \end{align}

\subsection{An H-CBF for the Double Integrator}
\label{subsec: double_integrator_CBF}
For the double integrator model above, an H-CBF for an agent $(\mathbf{p}_A,\mathbf{v}_A)$ and a (possibly moving) non-accelerating and non-rotating obstacle $(\mathbf{p}_O,\mathbf{v}_O)$ is given in
\cite{borrmann_control_2015} as
\begin{equation}
    h(\mathbf{x}) = \hat{\mathbf{n}}(\theta)^T (\mathbf{p}_A - \mathbf{p}_O) - \delta(\theta) - b.
    \label{eq_hij}
\end{equation}
Above, $\hat{\mathbf{n}}(\theta)$ is the unit normal vector of the hyperplane,  $\delta(\theta)$ is the required safety margin between the agent centre $\mathbf{p}_A$ and the obstacle centre $\mathbf{p}_O$ (with $\theta$-dependence for non-disc obstacles), and $b$ is the braking distance needed for the agent to accelerate/decelerate to a relative velocity of zero compared to the obstacle. The braking distance is given by
\begin{equation}
    b = 
    \begin{cases}
    \frac{(\hat{\mathbf{n}}(\theta)^T(\mathbf{v}_A - \mathbf{v}_O))^2}{2u_{\text{max}}},   & \text{for     } \hat{\mathbf{n}}(\theta)^T(\mathbf{v}_A - \mathbf{v}_O) < 0\\
    0,   & \text{for     } \hat{\mathbf{n}}(\theta)^T(\mathbf{v}_A - \mathbf{v}_O) \geq 0
\end{cases},
\label{eq_bij}
\end{equation}
since no braking is needed when the agent and obstacle are moving away from each other.

Note that above we have made a slight generalisation compared to \cite{borrmann_control_2015}, as they had no $\theta$-dependency and instead used $\hat{\mathbf{n}}(\theta)= \hat{\mathbf{n}}_{OA}$, the direction from obstacle to agent, and $\delta(\theta)=\delta$.

Finally, we note that the derivative of $h$ with respect to time, as needed in \eqref{eq: cont cons CBF}, for a fixed $\theta$ and with $b \neq 0$ is
\begin{equation} \label{eq: h dot exp}
\begin{split}
    \dot{h}(\mathbf{x}, \theta, \mathbf{u}) 
    &= \nabla h(\mathbf{x}, \theta) \dot{\mathbf{x}} \\
    &= \hat{\mathbf{n}}(\theta)^T(\mathbf{v}_A - \mathbf{v}_O - \mathbf{u} \frac{\hat{\mathbf{n}}(\theta)^T(\mathbf{v}_A-\mathbf{v}_O)}{u_{\text{max}}}).
\end{split}
\end{equation}

\section{Proposed Approach}
\label{section: Method}
In this section, we describe the proposed approach, show that the LRH-CBF is a generalisation of the MDH-CBF, and discuss the existence and conservativeness of H-CBFs. Then we investigate how to parameterise a family of H-CBFs, and discuss implementation details such as D-CBFs, optimisation strategies, and the handling of multiple obstacles.

\subsection{The Optimisation Problem}
The key idea of the proposed approach is to find the optimal choice of CBF with regard to the desired control.
To do this, the given CBF $h(\mathbf{x})$ is replaced by a family of CBFs $h(\mathbf{x},\theta)$ parameterised by $\theta \in \Theta$.
Then $\theta$ and $\mathbf{u}$ are optimised jointly to find the H-CBF that enables a safe control $\mathbf{u}$ as close as possible to $\mathbf{u}_{\text{des}}$. 
Thus, we replace Problem \ref{prob_CBF_QP} (CBF-QP) with the Least-Restrictive-CBF-Optimisation-Problem (LR-CBF-OP).

\begin{problem}[LR-CBF-OP]
\label{prob_P_CBF_QP}
\begin{align}
    (\mathbf{u}(\mathbf{x}), \theta) = \argmin_{\mathbf{u},\theta} & \ (\mathbf{u}-\mathbf{u}_{\text{des}})^T \mathbf{Q} (\mathbf{u}-\mathbf{u}_{\text{des}}) \\
    \text{s.t. } 
    & \ 
    \dot{h}(\mathbf{x},\theta,\mathbf{u}) \geq -\alpha(h(\mathbf{x},\theta)) \label{eq_CBF_constr}\\
    & \
    {h}(\mathbf{x},\theta) \geq 0  \label{eq_feasible}\\
    & \        
    \mathbf{u} \in \mathcal{U}, \ \theta \in \Theta.  \label{eq: CBF ineq} 
\end{align}
\end{problem}
Note that we have added line (\ref{eq_feasible}) to ensure that the chosen CBF keeps $\mathbf{x}$ feasible.
This corresponds to removing the hyperplane $P_5$ in Figure \ref{fig:page1}, for which the agent and the obstacle are on the same side.

\begin{theorem}[Safety and recursive feasibility]
    The LR-CBF-OP retains safety and recursive feasibility for any choice of $\Theta$ that includes the previous $\theta$.
\end{theorem}

\begin{proof}
    For a fixed $\theta$, the LR-CBF-OP turns into the regular CBF-QP, which is recursively feasible \cite{ames_control_2019}. Hence, each time the optimisation of the LR-CBF-OP is done,  $\theta_\text{old}$ will be a feasible solution. At each switch to $\theta_\text{new}$, $h(\mathbf{x}, \theta_\text{new}) \geq 0$ from \eqref{eq_feasible}, providing a new forward invariant set.
\end{proof}

% \begin{remark}
%     Paraphrasing the proof, we note that safety is always ensured for fixed $\theta$ for a valid choice of CBF. Optimising over $\theta$ might improve performance, but will not affect the safety guarantees.
% \end{remark}

Now this approach can, in principle, be applied to any family of CBFs. Still, we have chosen H-CBFs since they have closed-form solutions for second-order systems like the double integrator, and the degree of conservativeness has a clear dependence on their orientation.
By using the H-CBFs for double integrators in Section \ref{subsec: double_integrator_CBF}, we provide guarantees for both static ($\mathbf{v}_O=0$) and moving ($\mathbf{v}_O \neq 0$) obstacles.
Both cases are illustrated in Figure \ref{fig: Exp C spat}.

\begin{remark}[Computational Complexity]
\label{rem:non-convex}
A key advantage of the CBF approach in general is that Problem \ref{prob_CBF_QP} is a QP, with a quadratic objective function and linear constraints, which can be solved efficiently.
Looking at Problem~\ref{prob_P_CBF_QP} above, we note that the objective is still quadratic, but the constraints are not linear in $u$ and $\theta$. One way to address situations where low computational complexity is key is to make $\Theta$ a small finite set, having $N_\theta$ elements. Then we can pick a fixed $\theta\in \Theta$, and solve Problem \ref{prob_P_CBF_QP} (which is a QP again when $\theta$ is fixed). Iterate this over the $N_\theta$ different $\theta$, and pick the best solution.
In Section \ref{section: Experiments}, we explore the performance/computation tradeoff for this approach.

    % In \eqref{eq: h dot exp}, the time derivative of $\dot{\theta} = 0$. This is based on the notion that the optimisation of $\theta$ and $\mathbf{u}$ are separate. While finding the safe control $\mathbf{u}$ is done for a fixed $\theta$ in every timestep (which is fast, thanks to the problem then being a QP), the optimisation of $\theta$ can be performed less often to save computational time. Remember that safety is not dependent on $\theta$ being optimal, as optimising $\theta$ only decreases the conservativeness of the solution. This is similar to how the time-derivative $\dot{\hat{\mathbf{n}}}_{AO} = 0$ for the orthogonal CBF in \cite{borrmann_control_2015}, even though $\hat{\mathbf{n}}_{AO}$ will change direction over time. In both cases, the CBF-QP is solved for a fixed plane.

    % When inserting \eqref{eq: main CBF} and \eqref{eq: h dot exp} into \eqref{eq_CBF_constr}, we get a possibly non-convex function with respect to $\theta$.
    % Thus, Problem \ref{prob_P_CBF_QP} is only a QP for fixed $\theta$, and a general numerical solver such as CasADi\cite{andersson_casadi_2019} is needed when jointly optimising over $\theta$ and $\mathbf{u}$.
    % However, note that the $\theta$ from the previous time step is always a feasible option, so if computational resources are scarce, the fixed-$\theta$ QP can be solved every timestep, and the optimisation including $\theta$ can be done on a slower time-scale.
\end{remark}

\subsection{Generalisation of the Minimum Distance CBF}
First, we note that the proposed approach is in fact a generalisation of the very common 
distance-based CBFs used in \cite{borrmann_control_2015, wang_safety_2017, molnar_collision_2025, funada_collision_2025, thyri_reactive_2020}.
Let $\theta_{OA}$ be the parameter that makes $\hat{\mathbf{n}}(\theta_{OA})=\hat{\mathbf{n}}_{OA}$, a unit vector aligned with the shortest distance from the obstacle to the agent. Furthermore, let the set $\Theta$ be composed of this single element, $\Theta=\{\theta_{OA}\}$.
Then, as noted above, Equations (\ref{eq_hij}) and (\ref{eq_bij}) are identical to the ones in \cite{borrmann_control_2015}, and since $\theta$ is fixed, Problem \ref{prob_P_CBF_QP}
is identical to Problem \ref{prob_CBF_QP}. We call this special case MDH-CBF.

Given this observation, if the set $\Theta$ is chosen such that $\theta_{OA} \in \Theta$, the proposed approach LRH-CBF will never be more restrictive than MDH-CBF.
% Furthermore, as noted above, if $\Theta$ is a finite set of $N_\theta$ elements, the computations required are $N_\theta$ times those of MDH-CBF.

\subsection{Existence and Conservativeness}
In this section, we will first use results from convex analysis to show when an H-CBF of a given orientation exists, then analyse conservativeness by providing sufficient conditions for when an H-CBF can be found for a given safe trajectory, and provide an example of a very natural safe trajectory that would not be allowed by an MDH-CBF. 

% To investigate some properties of LRH-CBFs, the following two lemmas, inspired by the convex analysis results of Section~\ref{sec_convex_analysis}, state sufficient conditions for safety and conservativeness that are ensured by LRH-CBFs.

\begin{lemma}
\label{lemma:normal}
\textit{(An H-CBF for each orientation)}
    Given a compact obstacle  $O \subset \mathbb{R}^n$ with non-empty interior and a vector $\hat{\mathbf{n}}(\theta)$, there is an H-CBF $h(\mathbf{x},\theta)$ with normal $\hat{\mathbf{n}}(\theta)$ that is a supporting hyperplane to $\text{conv}(O)$, the convex hull of the obstacle.
\end{lemma}
\begin{proof}
    The set $\text{conv}(O)$ is closed, convex, bounded, and has non-empty interior. From \cite[Thm. 11.5 and Thm. 11.6]{rockafellarConvexAnalysis1970}, there is a supporting hyperplane with normal $\hat{\mathbf{n}}(\theta)$ that can be turned into an H-CBF $h(\mathbf{x},\theta)$ using (\ref{eq_hij}).
\end{proof}

\begin{lemma}\label{lemma: trajectory}
\textit{(Sufficient conditions for existence of a single H-CBF that does not interfere with a given safe trajectory)}
    Given an obstacle $O \subset \mathbb{R}^n$, and an agent with a trajectory $P = \{p_1, p_2, \ldots \} \subset \mathbb{R}^n$ ending with zero velocity.
    If the convex hull of the obstacle and the convex hull of the trajectory positions do not intersect, i.e., $\text{conv}(O) \cap \text{conv}(P) = \emptyset$, 
    then there exists a choice of $\theta$ and $\alpha$ such that the H-CBF $h(\mathbf{x},\theta)$ would have guaranteed the safety of $P$. 
\end{lemma}
\begin{proof}
    By \cite[Thm. 11.4]{rockafellarConvexAnalysis1970}, there exists a hyperplane separating the obstacle from the trajectory.
    By choosing $\theta$ in Lemma~\ref{lemma:normal} such that the normal $\hat{\mathbf{n}}(\theta)$ is equal to the normal of this hyperplane, the H-CBF will also be a separating hyperplane.
    Since the trajectory ends at zero velocity and never intersects the hyperplane, the agent's velocity is always safe with respect to hyperplane collisions, i.e., $h(\mathbf{x}, \theta)\geq 0$. 
    Now we show that (\ref{eq_CBF_constr}) is satisfied along the trajectory.
    Since the trajectory never intersects the hyperplane, if $h(\mathbf{x}, \theta)=0$ then the corresponding $u$ must satisfy $\dot h(\mathbf{x}, \theta, \mathbf{u}) \geq 0$, which satisfies (\ref{eq_CBF_constr}).
    If $h(\mathbf{x}, \theta)>0$, we can first pick an arbitrary class $\mathcal{K}_\infty$ $\alpha(x)$ function and then choose $\alpha_M(x)=M\alpha(x)$, with $M\in \mathbb{R}$ large enough to satisfy (\ref{eq_CBF_constr}) for $\alpha_M$.
   
\end{proof}

\begin{remark}
Note that Lemma \ref{lemma: trajectory} is sufficient for the existence of a \emph{single} H-CBF, while the core idea of the proposed approach is to iteratively find new H-CBFs throughout the execution. Thus, one might imagine a trajectory circling an obstacle, then there would (obviously) not be a single H-CBF proving safety, but the iterative choice of new H-CBFs during the execution could prove safety throughout the trajectory.
\end{remark}

The main motivation of this paper is that the LRH-CBF is less conservative than the commonly used MDH-CBF, while requiring computations of the same order of magnitude.
To support this claim, we first present a detailed example and then a set of simulations.

% \begin{figure}[t!]
%     \centering
%     \def\svgwidth{0.35\textwidth}\footnotesize
%     \import{Images/}{benifit_exampley.pdf_tex}
%     \caption{The geometry of Example 1. The agent $A$ moves along a trajectory with constant velocity $\mathbf{v}_A$, safely passing the static obstacle $O$. At the current time, $\hat{\mathbf{n}}_1$ defines a plane orthogonal to the vector connecting $A$ and $O$, while $\hat{\mathbf{n}}_2$ defines a plane which is parallel to $\mathbf{v}_A$.}
%     \label{fig: benefit example}
% \end{figure}

\begin{example}\label{exp: Lemma}
    Consider an agent $A$, with the double integrator dynamics of Section \ref{sec: DI model}, at position $\mathbf{p}_A$ moving at a velocity $\mathbf{v}_A$, and a stationary obstacle $O$ at position $\mathbf{p}_O$, as in Figure \ref{fig:page1}. Assume that the current desired control is $\mathbf{u}_{\text{des}} = 0$. In Figure \ref{fig:page1}, $A$ moves towards $G_1$ along a trajectory that will not collide with $O$, as long as $\mathbf{u}=0$. In regard to the MDH-CBF ($P_2$) defined by $\hat{\mathbf{n}}_2$, the $h$-value is given by
    \begin{equation}
        h(\mathbf{x}) = \hat{\mathbf{n}}_2^T(\mathbf{p}_A-\mathbf{p}_O)-\delta - \frac{(\hat{\mathbf{n}}_2^T\mathbf{v}_{A})^2}{2u_{\text{max}}}.
    \end{equation}
    Note that if $||\mathbf{v}_A||$ is large enough in comparison with $u_{\text{max}}$, the last term will switch the sign of $h$. Since a state that gives a negative $h$-value is considered part of the unsafe set, this means that the state of $A$ in Figure \ref{fig:page1} is not allowed by an MDH-CBF, even though the state is obviously safe. In contrast, the $h$-value for $A$ in regard to the hyperplane $P_4$ defined by $\hat{\mathbf{n}}_4$ is given by (with $\hat{\mathbf{n}}_4^T \mathbf{v}_A = 0$)
    \begin{equation}
    \label{eq_h}
        h(\mathbf{x}) = \hat{\mathbf{n}}_4^T(\mathbf{p}_A-\mathbf{p}_O)-\delta.
    \end{equation}
    With this choice of hyperplane, all magnitudes  $||\mathbf{v}_A||$ are allowed, as long as the direction of $\mathbf{v}_A$ is given by Figure~\ref{fig:page1}. Note that moving past an obstacle that is still partially in front of an agent, as in this example, is a common situation in crowded environments. 
    A simulation of the example above is illustrated in Figures \ref{fig: Exp A spatial} and \ref{fig: Exp A constraint}.
\end{example}

In the remainder of this section, we will first discuss different parameterisations of the H-CBF $h(\mathbf{x},\theta)$,
including the general parameterisation of the H-CBF, and the special cases of polygons and ellipses. Then, practical details needed for implementation are presented, including smooth approximation of $\delta(\theta)$, D-CBF correction, and the extension to handle multiple hyperplanes.

\subsection{H-CBFs for Different Obstacle Shapes}

% \begin{figure}[t!]
%     \centering
%     \def\svgwidth{0.35\textwidth}\footnotesize
%     \import{Images/}{geometry.pdf_tex}
%     \caption{The geometry of the parametrised H-CBF. The smaller circle represents the agent ($A$), and the larger circle represents the obstacle ($O$). Many of the vectors used to define the CBF hyperplane are shown as arrows.}
%     \label{fig: geometry}
% \end{figure}

The H-CBF, and in particular $\delta(\theta)$, is obviously dependent on the obstacle shape, but before we go into details, we note that the obstacle $O$ below is assumed to be a configuration space obstacle \cite{lavallePlanningAlgorithms2006} in the sense that it takes the shape of the agent into account (e.g. if the agent has a given radius, the obstacle is inflated by the corresponding amount and the agent is considered as a point).

\begin{lemma}
    For a closed, bounded obstacle $O\subset \mathbb{R}^n$, at position $\mathbf{p}_O$, with non-empty interior,  we can find an H-CBF (\ref{eq_hij}) with orientation $\hat{\mathbf{n}}(\theta)$ by solving
     \begin{equation}
     \label{eq_delta}
        \delta(\theta) = \max_{\mathbf{p} \in \text{conv}(O)} \hat{ \mathbf{n}}(\theta)^T \mathbf{p},
 \end{equation}
 where $\mathbf{p}$ is given in a coordinate system centred at $\mathbf{p}_O$.
\end{lemma}
\begin{proof}
    From \cite[Thm. 11.5 and Thm. 11.6]{rockafellarConvexAnalysis1970}, we know that there exists a point $\mathbf{p}^*$ and a  supporting hyperplane for every direction $\hat{ \mathbf{n}}(\theta)$ such that
 \begin{equation}
        \hat{ \mathbf{n}}(\theta)^T \mathbf{p}^* = \max_{\mathbf{p} \in S} \hat{ \mathbf{n}}(\theta)^T \mathbf{p}
 \end{equation}
 Let $\hat{ \mathbf{n}}_\perp(\theta)$ be orthogonal to $\hat{ \mathbf{n}}(\theta)$, then we can compose
 $\mathbf{p}^*$ out of two components $\mathbf{p}^*= \delta(\theta)  \hat{ \mathbf{n}}(\theta) +\epsilon(\theta) \hat{ \mathbf{n}}_\perp(\theta)$, where $\delta(\theta)$ is the distance from $\mathbf{p}^*$ to a hyperplane orthogonal to $\hat{ \mathbf{n}}(\theta)$ passing through $\mathbf{p}_O$. Letting $S=\text{conv}(O)$ and $\mathbf{p}^*$ as above we get 
  $\max_{\mathbf{p} \in \text{conv}(O)} \hat{ \mathbf{n}}(\theta)^T \mathbf{p}= \hat{ \mathbf{n}}(\theta)^T \mathbf{p}^*= \delta(\theta)$ as above.
\end{proof}

\begin{remark}
    Note that the optimisation in (\ref{eq_delta}) only needs to be solved once for each $\theta$, for a given obstacle, unlike the optimisation in Problem \ref{prob_P_CBF_QP}, which is solved each time step.
\end{remark}

\begin{remark}
    For polygonal obstacles, the supporting hyperplane has to intersect at least one vertex, thus, the optimisation simplifies to 
    \begin{equation}
        \delta(\theta) = \max_{\mathbf{p} \in \{p_1, p_2, \ldots, p_K\}} \hat{ \mathbf{n}}(\theta)^T \mathbf{p},
    \end{equation}
    where $\{p_1, p_2, \ldots, p_K\}$ is the vertex set.
\end{remark}

% \begin{lemma}[Parameterised Boundaries]
%     If a parameterisation of the boundary of $conv(O)$ is available in terms of \begin{equation}
%     \partial \text{conv}(O) = \{\mathbf{x} \vert \mathbf{x} = \mathbf{p}_O + \hat{\mathbf{n}}(\phi) r_O(\phi), \phi \in [0, 2\pi)\}
% \end{equation}
% where $\mathbf{p}_O \in \text{conv}(O)$,
% and $r_O:  \mathbb{R} \to \mathbb{R}$ is the radial distance, see Figure \ref{fig: obstacle}, then we have
% \begin{equation} \label{eq: ang safety}
%     \delta(\theta) = \max_\phi\{\hat{\mathbf{n}}(\theta)^T\hat{\mathbf{n}}(\phi)r_O(\phi)\}.
% \end{equation}

% \end{lemma}
% \begin{proof}
%     Direct substitution.
% \end{proof}

% Consider a generally shaped obstacle $O$, as in Figure \ref{fig: obstacle}. The convex hull of the obstacle is  denoted $\text{conv}(O)$, and due to convexity,  we can describe the perimeter of $\text{conv}(O)$ as

%possible since every convex set can be described as a star domain.

\begin{figure}[t!]
    \centering
    \def\svgwidth{0.45\textwidth}\footnotesize
    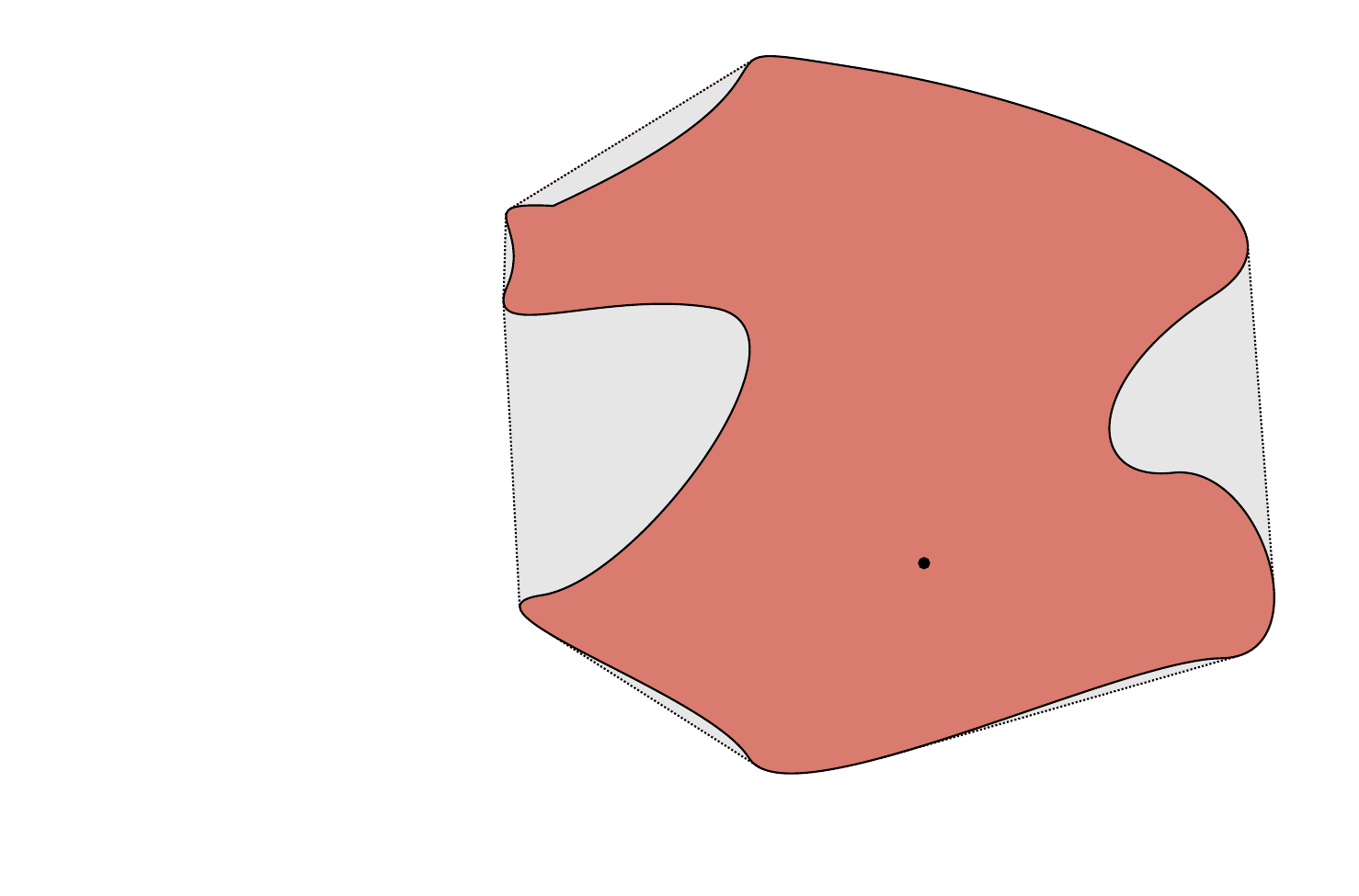
    \caption{The geometry of a general 2D obstacle. The red area is the obstacle $O$, and combined with the light grey areas, we get $\text{conv}(O)$. Notice how the braking distance $b$ added to the supporting hyperplane defines the boundary of the set of safe positions, given the current velocity $\mathbf{v}_A$.}
    \label{fig: obstacle}
\end{figure}

% With this  representation, we can introduce the angle-dependent safety distance function $\delta_{O}(\theta)$  as 
% \begin{equation} \label{eq: ang safety}
%     \delta_{O}(\theta) = \max_\phi\{\hat{\mathbf{n}}(\theta)^T\hat{\mathbf{n}}(\phi)r_O(\phi)\}.
% \end{equation}
% This describes the maximum distance that $\text{conv}(O)$ extends out from the hyperplane crossing $\mathbf{p}_O$ with the normal vector $\hat{\mathbf{n}}(\theta)$. Then, the total safety distance needed between the center of the circular  agent $A$ (radius $r_A$) and the obstacle point $\mathbf{p_O}$ as a function of $\theta$ is given by 
% \begin{equation} \label{eq: arb obst}
%     \delta(\theta) = r_A + \delta_O(\theta).
% \end{equation}
% Inserting \eqref{eq: arb obst}  into \eqref{eq: main CBF}, we get the parameterised H-CBF for any obstacle. 
% Note that the optimisation in \eqref{eq: arb obst} only needs to be done once for each obstacle, as it only depends on the shape of the obstacle.

% \begin{remark}
%     Due to the symmetry between the agent and the obstacles, different shapes of the agent can be handled similarly to the obstacles. For the rest of this paper, the agent will be considered disk-shaped for the sake of simplicity.
% \end{remark}

\begin{lemma}[Ellipse Boundaries]
    For an ellipse given in the form
    \begin{equation} \label{eq_ellipse}
    \textbf{r}(\zeta) =
        \begin{bmatrix}
            a \ \text{cos}(\zeta), & 
            b \ \text{sin}(\zeta)
        \end{bmatrix},
    \end{equation}
    where $a$ is the major axis, $b$ is the minor axis, and the angle $\zeta$ is called the eccentric anomaly, the 
    optimisation problem in \eqref{eq_delta} has the solution
    \begin{equation}
        \delta(\theta) = a \ \text{cos}(\theta)\text{cos}(\zeta_{\text{max}}) + b \ \text{sin}(\theta)\text{sin}(\zeta_{\text{max}}), 
        \label{eq: delta_ellips}
    \end{equation}
    where $\zeta_{\text{max}} = \text{arctan2}(b \ \text{sin}(\theta), a \ \text{cos}(\theta))$.
    % \begin{equation} \label{eq_a_max}
    %     \zeta_{\text{max}} = \text{arctan2}(b \ \text{sin}(\theta), a \ \text{cos}(\theta)).
    % \end{equation}
\end{lemma}

\begin{proof}
    Inserting \eqref{eq_ellipse} into \eqref{eq_delta}, we get
    %\begin{equation}
        $\delta(\theta) = \max_{\zeta}\{\hat{\mathbf{n}}(\theta)^T\textbf{r}(\zeta)\}$.
    %\end{equation}
    Setting the derivative to zero gives
%    \begin{align}
        $\frac{\partial}{\partial \zeta}(\hat{\mathbf{n}}(\theta)^T\textbf{r}(\zeta)) 
        = -\text{cos}(\theta) a \text{sin}(\zeta) + \text{sin}(\theta)b\text{cos}(\zeta) =0,$ 
 %   \end{align}
    which in turn gives the
    $\zeta_{\text{max}}$ above,
    corresponding to (\ref{eq: delta_ellips}).
\end{proof}
The result can be extended to an arbitrary ellipse rotated by an angle $\beta$.

\subsection{Closed-Form Over-approximations of Delta}
For a smooth closed-form approximation of $\delta(\theta)$, one can create an approximation using a truncated Fourier series,
\begin{equation}
    \Tilde{\delta}(\theta) = \frac{a_0}{2} + \sum^{N}_{n = 1}(a_n\text{cos}(n \theta) + b_n \text{sin}(n \theta)) + c,
\end{equation}
where $c$ is a scalar safety margin that can be used to make sure the approximation is always an over-approximation.

\subsection{Discrete Time System CBFs}
The relation between CBFs and D-CBFs, using the $h$-function in \eqref{eq_hij}, is described in the following lemma:

\begin{lemma}[D-CBF]
    With the $h$-function in \eqref{eq_hij}, the corresponding D-CBF constraint becomes
    \begin{equation} \label{eq: DI DCBF}
        \dot{h}(\mathbf{x}_k) - \Delta t \frac{(\hat{\mathbf{n}}(\theta)^T(\mathbf{u}_k))^2}{2 u_\text{max}} \geq -\alpha( h(\mathbf{x}_k)),
    \end{equation}
    where $\alpha(x) = x\gamma / \Delta t$, and $h$ and $\dot{h}$ are the same as in the continuous case.
\end{lemma}

\begin{proof}
    Expanding \eqref{eq_hij} at $\mathbf{x}_{k+1}$ and inserting into \eqref{eq: DCBF alt cnstr} results in \eqref{eq: DI DCBF}.
    % Calculate the progression of the state after one timestep,
    % \begin{align} \label{eq: Proof DI}
    %     h(\mathbf{x}_{k + 1}) & =  h(\mathbf{f}(\mathbf{x}_{k}) + \mathbf{g}(\mathbf{x}_k) \mathbf{u}_k) \nonumber \\ & = \hat{\mathbf{n}}(\theta)^T(\mathbf{p}_{A, k} + \Delta t \mathbf{v}_{A,k} - \mathbf{p}_{O,k} - \Delta t \mathbf{v}_{O, k}) \nonumber \\ & -\delta - \frac{(\hat{\mathbf{n}}(\theta)^T(\mathbf{v}_{A,k} + \Delta t \mathbf{u}_{A,k} - \mathbf{v}_{O,k}))^2}{2u_{\text{max}}} \nonumber \\
    %     & = h(\mathbf{x}_k) + \Delta t \dot{h}(\mathbf{x}_k) - \frac{(\Delta t \hat{\mathbf{n}}(\theta)^T\mathbf{u}_k)^2}{2 u_\text{max}}.
    % \end{align}
    % Inserting \eqref{eq: Proof DI} into \eqref{eq: DCBF alt cnstr} results in \eqref{eq: DI DCBF}.
\end{proof}

% Notice how the only difference between the continuous CBF and the D-CBF is the addition of the negative square term. First, this means that the problem reduces to a QCQP for a fixed $\theta$. Secondly, this means that the CBF constraint will be slightly more conservative than a naive Euler approximation of \eqref{eq: cont cons CBF}, as it decreases the value on the left side.

\subsection{Choosing $\Theta$} \label{sub: theta}

The choice of $\Theta$ can be done arbitrarily, but for symmetry reasons, we will often let the middle element be $\theta_{OA}$, which marks the direction of the closest point from the obstacle to the agent. 
Then, we let $\Theta$ be symmetric around $\theta_{OA}$.
Note that any $\theta \notin (\theta_{OA} - \pi / 2, \theta_{OA} + \pi / 2)$ will automatically break the condition \eqref{eq_feasible}, since that would render the agent and the obstacle on the same side of the hyperplane.
For a non-discrete set $\Theta$, it might thus make sense to pick the interval 
$$\Theta= (\theta_{OA} - \pi / 2, \theta_{OA} + \pi / 2).$$
In practice, the zone of valid $\theta$ choices is even smaller, depending on the obstacle size and distance.

For a finite set $\Theta$, we want to numerically explore the effect of the number of elements $N_\theta$ and the maximum spread,  $\sigma_\theta$, measured from $\theta_{OA}$.
Thus we let
\begin{equation}
\label{eq_theta_set}
\Theta_{N_\theta}^{\sigma_\theta} = \{\theta_{OA} - \sigma_\theta, ..., \theta_{OA}, ..., \theta_{OA} + \sigma_\theta, \theta_\text{old} \},
\end{equation}
where the $N_\theta$ hyperplanes are equally spread between $\theta_{OA} - \sigma_\theta$ and $\theta_{OA} + \sigma_\theta$, and $\theta_\text{old}$ is the previous choice of $\theta$.
Note that $\theta_\text{old}$ is needed for recursive feasibility.
\begin{remark}
    While grid search works well for $\theta \in \mathbb{R}$ (a single obstacle in a planar world), the computational complexity scales exponentially as $N_\theta^n$ when $\theta \in \mathbb{R}^n$ ($n$ obstacles, see below). Hence, conventional numerical solvers such as CasADi with IPOPT are better suited for higher-dimensional $\theta$ \cite{andersson_casadi_2019}.
\end{remark}

\begin{remark}
    By including $\theta_{OA}$ in $\Theta$, the gap $||\mathbf{u}- \mathbf{u}_{\text{des}} ||$ can never be larger in the proposed approach than in MDH-CBF.
\end{remark}

\subsection{Multiple Obstacles}
It was noted in \cite{breeden_compositions_2023} that the combination of multiple CBFs from different obstacles might cause problems. For example, an agent with input constraints, moving towards two nearby obstacles with accompanying CBFs might be in a state deemed safe by both CBFs individually, while the safe control sets of the two CBFs do not overlap, i.e., there is no control choice that is safe in relation to both obstacles.
To handle such issues, a slight modification to \eqref{eq_bij} is needed, where we introduce a scalar factor $\xi\in [0,1]$ before $u_{\max}$ capturing the fact that full acceleration cannot be applied in several different directions simultaneously. For example, if the two directions $\hat{\mathbf{n}}(\theta_1),\hat{\mathbf{n}}(\theta_2)$ are orthogonal, the available acceleration in both directions is reduced by $\xi = 1/\sqrt{2}$ in Equation (\ref{eq_xi}).

\begin{theorem}
    Let $\mathcal{N}_a = \{\hat{\mathbf{n}}(\theta_1), ..., \hat{\mathbf{n}}(\theta_n)\}$ be the fixed unit normals of all the H-CBFs that can be simultaneously active, i.e., having $\hat{\mathbf{n}}(\theta_i)^T(\mathbf{v}_A - \mathbf{v}_{Oi}) < 0$. Then define
    \begin{equation}
    \label{eq_xi}
        \xi = \text{dist}(0, \text{conv}(\mathcal{N}_a)).
    \end{equation}
    If $\xi > 0$ and $b$ in \eqref{eq_bij} is replaced with
    \begin{equation}
        b = 
    \begin{cases}
    \frac{(\hat{\mathbf{n}}(\theta)^T(\mathbf{v}_A - \mathbf{v}_O))^2}{2 \xi u_{\text{max}}},   & \text{for     } \hat{\mathbf{n}}(\theta)^T(\mathbf{v}_A - \mathbf{v}_O) < 0\\
    0,   & \text{for     } \hat{\mathbf{n}}(\theta)^T(\mathbf{v}_A - \mathbf{v}_O) \geq 0
    \end{cases},
    \end{equation}
    then for every state $\mathbf{x} \in \{\mathbf{x}: h_i(\mathbf{x}) \geq 0, \forall i \in \{1, ..., n\}\}$ there exists a $\mathbf{u}$ with $||\mathbf{u}|| \leq u_\text{max}$ that satisfies all the $n$ H-CBFs simultaneously. Hence, the CBF-QP (and the LR-CBF-OP with $\theta_\text{old} \in \Theta$) remains recursively feasible with multiple hyperplanes.

    % Consider $n \geq 2$ hyperplanes with orientations $\theta_i, i \in \{1, ..., n\}$.
    % To retain recursive feasibility, it is sufficient to replace the $b$ in the nominal H-CBF \eqref{eq_bij} with
    % \begin{equation}
    %     b = 
    % \begin{cases}
    % \frac{(\hat{\mathbf{n}}(\theta)^T(\mathbf{v}_A - \mathbf{v}_O))^2}{2 \xi u_{\text{max}}},   & \text{for     } \hat{\mathbf{n}}(\theta)^T(\mathbf{v}_A - \mathbf{v}_O) < 0\\
    % 0,   & \text{for     } \hat{\mathbf{n}}(\theta)^T(\mathbf{v}_A - \mathbf{v}_O) \geq 0
    % \end{cases},
    % \end{equation}
    % where $\xi = \sin (\Tilde{\psi} / 2$). The minimum hyperplane intersection angle $\Tilde{\psi}$ is calculated as $\Tilde{\psi} = \min \psi_{ij}$, with $i \neq j$, $i, j \in \{1, ..., n\}$, and $\psi_{ij} = \pi - \arccos(\hat{\mathbf{n}}(\theta_i)^T\hat{\mathbf{n}}(\theta_j))$.
\end{theorem}

\begin{proof}
    The H-CBFs with $\hat{\mathbf{n}}(\theta_i) \notin \mathcal{N}_a$ have $b = 0$, and satisfy \eqref{eq_CBF_constr} for any $\mathbf{u}$. 
    For an H-CBF with $\hat{\mathbf{n}}(\theta_i) \in \mathcal{N}_a$ and the modified $b_i$, $\hat{\mathbf{n}}(\theta_i)^T \mathbf{u} \geq \xi u_\text{max}$ implies \eqref{eq_CBF_constr}, with equality at $h(\mathbf{x}) = 0$. 
    Let $\mathbf{u} = \hat{\mathbf{n}}(\psi) u_\text{max}$, reducing feasibility to $\max_\psi \min_{\hat{\mathbf{n}}(\theta_i) \in \mathcal{N}_a} \hat{\mathbf{n}}(\theta_i)^T \hat{\mathbf{n}}(\psi)\geq \xi$.
    Assuming $0 \notin \text{conv}(\mathcal{N}_a)$,  the point in $\text{conv}(\mathcal{N}_a)$ closest to the origin is given by $\mathbf{p}^* = \hat{\mathbf{n}}(\psi_a) \xi$ with $\xi > 0$.
    For any $\mathbf{z} \in \text{conv}(\mathcal{N}_a)$, we have $(\mathbf{z} - \mathbf{p}^*)^T\mathbf{p}^* \geq 0$ from convexity.
    Rearranging, it follows that $\hat{\mathbf{n}}(\theta_i)^T \hat{\mathbf{n}}(\psi_a)\geq~\xi \ \forall \ i \in~\{1, ..., n\}$.
    Hence $\mathbf{u} = \hat{\mathbf{n}}(\psi_a) u_\text{max}$ is a feasible $\mathbf{u}$ for all H-CBFs with $\hat{\mathbf{n}}(\theta_i) \in \mathcal{N}_a$. 
    Forward invariance then follows from \cite{ames_control_2019}.
\end{proof}

\begin{remark}
    The inclusion of $\xi$ couples the $h$-value of the H-CBFs, as it depends on all $\theta_i$. This does not affect the MDH-CBF, but it will have an influence on all methods that include $h$ in the optimisation.
\end{remark}

% \begin{proof}
%     To ensure feasibility, the intersection point $\mathbf{u}_{ij}$ must ensure $||\mathbf{u}_{ij}||_2 \leq \mathbf{u}_\text{max}$. The most critical condition is found when $h_i(\mathbf{x}) = h_j(\mathbf{x}) = 0$, corresponding to two hyperplanes defined by the points $\mathbf{u}_i = \hat{\mathbf{n}}(\theta_i) \xi u_\text{max}$ and corresponding normal vectors, as illustrated in figure \ref{fig: multi obs}. From the right-angled triangle, it follows that $\xi = \sin (\psi_{ij} / 2)$.
% \end{proof}

\begin{remark}
    This condition might seem conservative, but is essential to safety when moving inside a structure such as a building, where obstacles (walls) meet. This case is overlooked in many papers. The construction used here builds upon the ideas in \cite{breeden_compositions_2023}. 
\end{remark}

\section{Simulations}
\label{section: Experiments}

As noted above, the proposed approach is a generalisation of hyperplane CBFs used in, e.g.
\cite{borrmann_control_2015, wang_safety_2017, molnar_collision_2025, funada_collision_2025, thyri_reactive_2020}, here denoted MDH-CBF.
It is simple to use and understand, and computationally lightweight. The contribution is that it improves performance while still being straightforward. Therefore, we will provide experiments that support this claim, but not compare it to computationally heavy approaches using, e.g., machine learning or dynamic programming. 
To enable comparison with the MDH-CBF and VO-CBF methods used in \cite{borrmann_control_2015, wang_safety_2017, roncero_multi-agent_2025-1}, we consider a double integrator with acceleration constraints.
The simulations are conducted using NumPy, and the numerical solver CasADi uses IPOPT \cite{andersson_casadi_2019}. Whenever possible, the optimisation is warm-started with the previous time step value $\theta_\text{old}$.

We will present results for five different methods listed below. The first two are from the literature, and the last three are different variations on the proposed approach. 

MDH-CBF: Minimum Distance H-CBF, where the hyperplane corresponding to the minimum distance from the agent to the obstacle is chosen  \cite{borrmann_control_2015, wang_safety_2017, molnar_collision_2025, funada_collision_2025, thyri_reactive_2020}.

VO-MDH-CBF: Velocity Obstacle MDH-CBF, where velocity obstacles are used to guide the agent when far away from the obstacles, while an MDH-CBF is used to ensure safety in close encounters \cite{roncero_multi-agent_2025-1}. 

LRH-CBF: The main approach presented in this paper.

VO-LRH-CBF: as above, with velocity obstacles added.

MHH-CBF: Maximum H-value H-CBF, where the H-CBF is optimised over $\theta$ to maximise $h(\mathbf{x}, \theta)$ in (\ref{eq_hij}).

% Thus, we will compare the proposed method

% To evaluate the performance of our proposed method, two popular methods from the literature will be considered.
% \begin{itemize}
%     \item Minimum Distance H-CBF (MDH-CBF), where the hyperplane corresponding to the minimum distance from the agent to the obstacle is chosen (or equivalently, the H-CBF that maximises the distance to the agent) \cite{borrmann_control_2015, funada_collision_2025}.
%     \item Velocity Obstacle MDH-CBF (VO-MDH-CBF), where velocity obstacles are used to guide the agent when far away from the obstacles, while an MDH-CBF is used to ensure safety in close encounters \cite{roncero_multi-agent_2025}. 
% \end{itemize}
% Additionally, we will consider two methods that we consider natural extensions of previous work when considering the optimal choice of H-CBFs.
% \begin{itemize}
%     \item Maximum H-value H-CBF (MHH-CBF), where the H-CBF is optimised to maximise $h(\mathbf{x})$.
%     \item VO-LRH-CBF, where the VO is combined with an LRH-CBF.
% \end{itemize}

Before comparing the different methods, we will compare two versions of LRH-CBF: continuous optimisation over an interval of $\theta$ values (C-LRH-CBF) with grid search over a discretised set of $\theta$ values (G-LRH-CBF).

% In this section, the performance of the proposed approach is illustrated through four examples of increasing complexity, using a 2D double integrator with actuation constraints.
% First, we compare the performance of a finite $\Theta$ with equidistant values explored by exhaustive grid search (G-LRH-CBF) and an infinite set (interval)  $\Theta$ explored using CasADi's nonlinear optimization (C-LRH-CBF).

% Then, the performance of the C-LRH-CBF is compared with state-of-the-art methods on increasingly complex examples. 

% As described above, the contribution of this paper is an approach that has a computational complexity that is very low, and comparable to the highly popular OH-CBF  \cite{borrmann_control_2015, wang_safety_2017, molnar_collision_2025, funada_collision_2025, thyri_reactive_2020}, while giving better performance due to being less restrictive.

% Thus, we compare the Continuous and Discrete LRH-CBFs (C- and G-LRH-CBFs) with different choices of $\Theta$, as well as with the OH-CBF from \cite{borrmann_control_2015}. 

 \begin{figure}[!h]
    \centering
    \def\svgwidth{0.45\textwidth}\footnotesize
    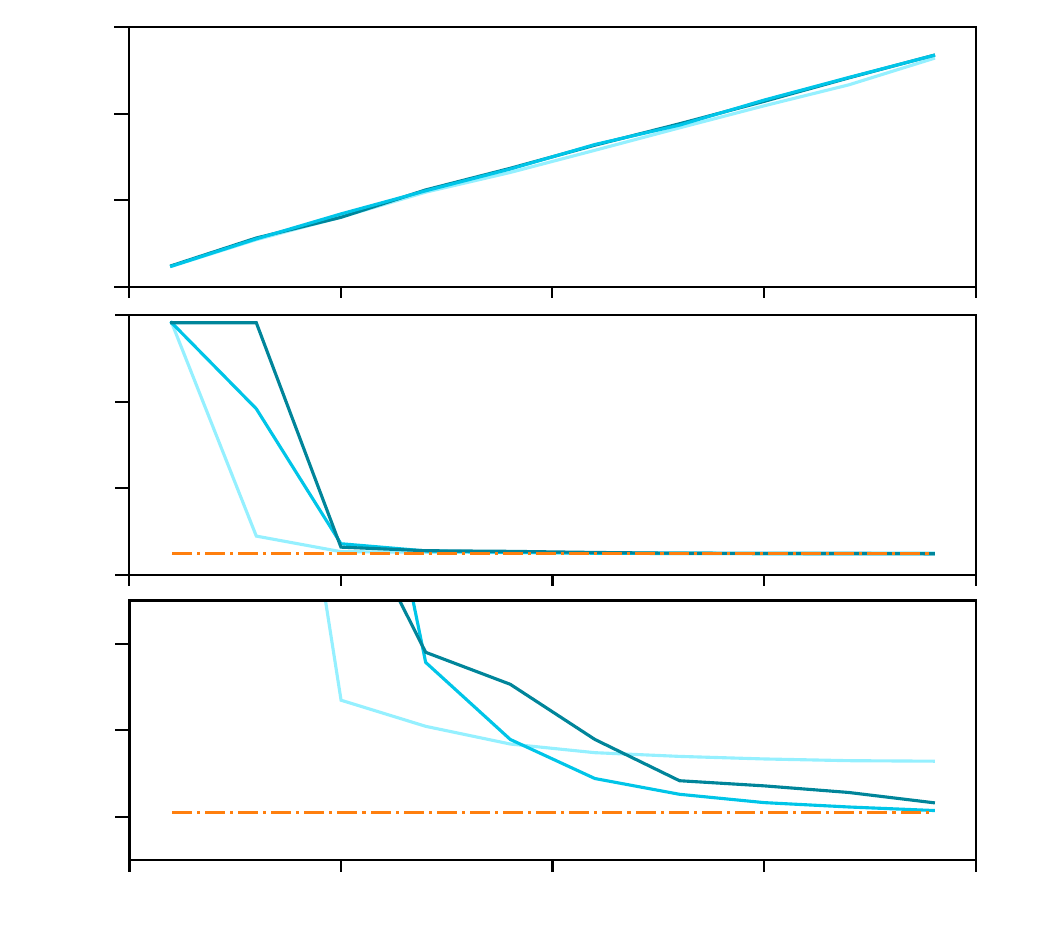
    \caption{The mean computational time and the aggregate control intervention $||\mathbf{u} - \mathbf{u}_\text{des}||$, over 10000 random states in the proximity of a circular obstacle, as a function of $N_\theta$ in Equation (\ref{eq_theta_set}). The C-LRH-CBF is used as a baseline and does not depend on the resolution $N_\theta$. 
    Upper: the linear relationship between $N_\theta$ and mean computational time, giving us the expected time $t(N_\theta) = 6.8 N_\theta$ $\mu$s. The mean time to solve a step with the CasADi solver is 7.2 ms, too long to be shown in the same plot.
    Middle: the full range in $||\mathbf{u} - \mathbf{u}_\text{des}||$. Lower: a zoomed-in version. Note that no improvement is gained when going from $N_\theta = 1$ to $N_\theta = 3$ for $\Theta^{0.5 \pi}$, which is an expected result from Section \ref{sub: theta}. Increasing the resolution from $N_\theta = 1$ quickly decreases $||\mathbf{u} - \mathbf{u}_\text{des}||$. At resolutions $N_\theta < 9$, a smaller spread of $\Theta^{0.3 \pi}$ performs best, while at higher resolutions the wider spread $\Theta^{0.4 \pi}$ performs best.}
    \label{fig: num res diff}
    \end{figure}

\begin{example} \label{exp: num res}
    In this example, the influence of the choice of $\Theta$ on the gap $||\mathbf{u} - \mathbf{u}_\text{des}||$ is investigated, as well as the computational load.
    To assess performance, 10000 states and desired controls were randomly sampled in the vicinity of a circular obstacle. Three choices of the spread $\sigma_\theta$ for the G-LRH-CBF were examined over a range of resolutions $N_\theta$, and plotted alongside the solution from the C-LRH-CBF (warm-started with the solution of $\Theta^{0.4 \pi}_{19}$). The metrics that were considered were the aggregated control intervention $||\mathbf{u} - \mathbf{u}_\text{des}||$ and the mean computation time over 10000 samples. The results are presented in Figure \ref{fig: num res diff}. From the middle and lower plots in Figure \ref{fig: num res diff}, it is apparent that a significant improvement is achieved when increasing the resolution with only a few hyperplanes. 
    % This implies that without much additional computational time, the difference $||\mathbf{u} - \mathbf{u}_\text{des}||$ is greatly reduced compared to the MDH-CBF (which corresponds to $N_\theta = 1$). 
    Looking at the zoomed-in performance in the lower part of Figure \ref{fig: num res diff}, we see that the choice of $\sigma_\theta$ has a major impact on convergence towards the C-LRH-CBF performance. While a small spread $\sigma_\theta = 0.3 \pi$ is good for resolutions $N_\theta < 9$, better performance is achieved with $\sigma_\theta = 0.4 \pi$ for resolutions $N_\theta > 9$. The time complexity, shown in the top plot in Figure \ref{fig: num res diff},  is linear in the choice of $N_\theta$, as expected. We also note that the approach is indeed quite fast, using an Intel Core Ultra 9 185H with WSL, the time per step is approximately $6.8  N_\theta \times  10^{-6}$ s. For the C-LRH-CBF, each step takes about $7.2 \times 10^{-3}$ s.

    % \begin{figure}[!t]
    % \centering
    % \includegraphics[width=0.99\columnwidth]{Images/DI_Res_vs_diff_ex_close.pdf}
    % \caption{Zoomed-in version of Figure \ref{fig: num res diff}. Notice how $\Theta^{0.3 \pi}$ performs best of the three discrete $\Theta$ at low resolutions, while $\Theta^{0.4 \pi}$ becomes better for $N_\theta > 9$. As the resolution is increased, $\Theta^{0.4 \pi}$ converges towards the continuous solution.}
    % \label{fig: num res diff close}
    % \end{figure}
    
    % \begin{figure}[!t]
    % \centering
    % \def\svgwidth{0.45\textwidth}\footnotesize
    % \import{Images/}{ExpOptTime.pdf_tex}
    % \caption{The mean computational time over 10000 steps as a function of $N_\theta$. At max resolution, $N_\theta = 19$, a step takes about 130 $\mu$s. As the relationship between time and $N_\theta$ is linear and independent of $\sigma_\theta$, the expected time of a single step is approximately $6.8  N_\theta$ $\mu$s. The mean time to solve a step with the continuous solver is 7.2 ms, too long to be shown in the same plot.}
    % \label{fig: num res time}
    % \end{figure}
\end{example}

\begin{remark}
    Note that Figure~\ref{fig: num res diff} also gives us a quantitative comparison between C-LRH-CBF and MDH-CBF, since $N_\theta=1$ in (\ref{eq_theta_set}) gives the MDH-CBF. We can see that for $N_\theta=1$, MDH-CBF has an aggregated control intervention $||\mathbf{u} - \mathbf{u}_\text{des}||$ over 10000 samples of approximately 150, while C-LRH-CBF is approximately 12. Thus, we see that the control intervention $||\mathbf{u} - \mathbf{u}_\text{des}||$ on average is a factor of 12 smaller for C-LRH-CBF.
\end{remark}

\begin{remark}
    While grid search works well for a single obstacle, time complexity grows exponentially with the number of hyperplanes. Hence, C-LRH-CBF is preferred when considering multiple obstacles.
\end{remark}

\begin{example}\label{exp: A}
    In this example, which is a numerical version of Example \ref{exp: Lemma}, the agent moves on a safe straight line with constant velocity, passing a static obstacle, as seen in Figure~\ref{fig: Exp A spatial}.
    Parts of the path are, however, \emph{not considered safe} (feasible) with respect to an MDH-CBF, as seen in Figure \ref{fig: Exp A constraint}.
    For $t\in [1.4, 2.0]$, the $h$-value is negative, which means that the state is considered unsafe.
    No matter the choice of $\alpha$, this safe trajectory can not be followed without braking when using an MDH-CBF.
    In contrast, the $h$-value for the C- and G-LRH-CBFs remains above zero throughout the trajectory, which is also true for the MHH-CBF.
    This implies that the trajectory can be followed given the right choice of $\alpha$, which shows that optimising over $\theta$ can result in a less restrictive H-CBF.

    \begin{figure}[!t]
    \centering
    \def\svgwidth{0.45\textwidth}\footnotesize
    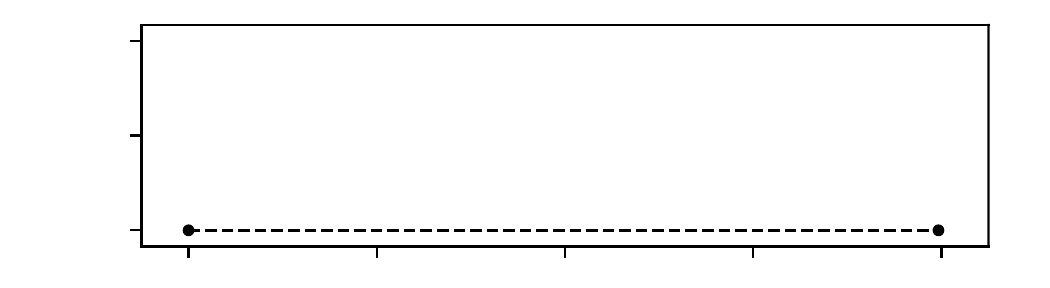
    \caption{The setup of Example \ref{exp: A}. 
    The agent moves on a straight line at constant velocity ($\mathbf{v}=[2, 0]^T$ m/s, $\mathbf{u}_{\text{des}} = 0$), safely passing an obstacle (the dotted rounded triangle corresponds to the configuration space of the obstacle). 
    % The corresponding $h$-value for the four different cases is shown in Figure \ref{fig: Exp A constraint}.
    }
    \label{fig: Exp A spatial}
    \end{figure}
    
    \begin{figure}[!t]
    \centering
    \def\svgwidth{0.45\textwidth}\footnotesize
    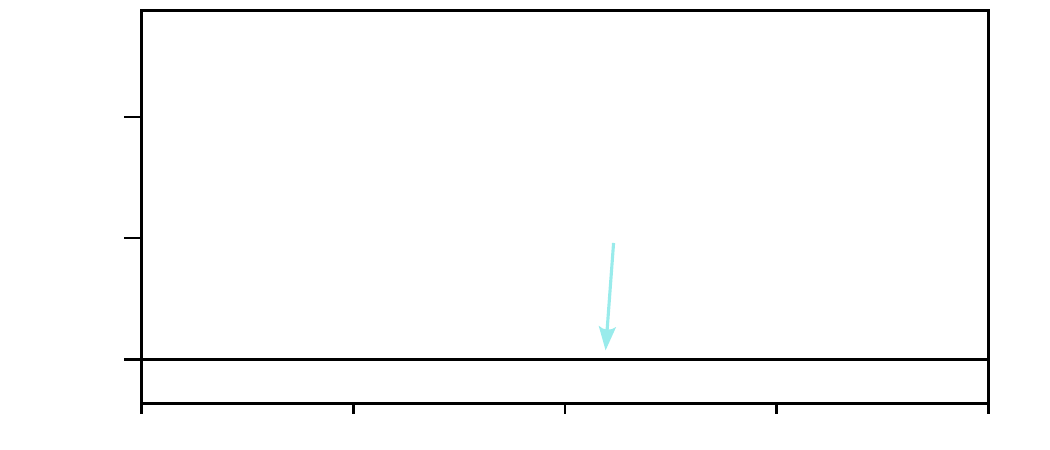
    \caption{The $h$-value for the fixed trajectory of the agent in Example~\ref{exp: A}. 
    Note that this safe trajectory would not be allowed by the MDH-CBF of \cite{borrmann_control_2015} (dashed), as the $h$-value is negative for $t\in [1.4, 2.0]$, right before passing the obstacle.
    This is independent of the choice of $\alpha$.
    However, since the $h$-value is positive during the entire trajectory for the other methods, there exists an $\alpha$ that allows their trajectories.
    % Note that some of the LRH-CBFs only deviate from the OH-CBF when it is needed (to avoid restricting the control), i.e., when the $h$-value of the OH-CBF turns negative around $t=1.5$.
    }
    \label{fig: Exp A constraint}
    \end{figure}
\end{example}

\begin{example} \label{exp: B}
    In this example, the C-LRH-CBF is compared to the other methods in a scenario with a single rectangular obstacle moving at constant velocity. The setup is run 500 times, with the obstacle's velocity being randomly sampled with the constraint that it intersects the nominal trajectory (causing a potential collision) for some $t \in [5, 7]$s. 
    The agent starts at rest and is driven by a PD controller that provides $\mathbf{u}_{\text{des}}$. 
    The results are presented in Table \ref{tab: stats}. Notice how the LRH-CBF-based methods reach the goal faster while intervening less with the desired control.
    
    One sample scenario from Example \ref{exp: B}  is seen in Figure \ref{fig: Exp B spatial}, with corresponding data in Figure~\ref{fig: Exp B distance}.
    As can be seen from $||\mathbf{u}-\mathbf{u}_{\text{des}}||$ in Figure \ref{fig: Exp B distance} (upper), the C-LRH-CBF intervenes later and less than all the other methods.
    This results in the C-LRH-CBF reaching within 1 m of the goal first, with the MHH-CBF and the VO-LRH-CBF close behind, see Figure \ref{fig: Exp B distance} (lower).
    Furthermore, looking at the paths in Figure \ref{fig: Exp B spatial}, we see that the C-LRH-CBF and MHH-CBF are allowed to move closer to the obstacles than the other methods. The VO-CBF methods start steering away from the obstacle early due to the VO component. The VO-LRH-CBF is almost as fast as the C-LRH-CBF, while taking a longer path.

    \begin{remark}
        Although the MHH-CBF performs similarly to the LRH-CBF in this experiment, there is a qualitative difference. Consider the situation in Figure \ref{fig:page1}: the H-CBF maximising $h$ is $P_3$ or $P_4$, while if $\mathbf{u}_\text{des}$ drives the agent towards the goal $G_2$, the LRH-CBF will be $P_1$ or $P_2$. The difference is that the MHH-CBF only considers the current state, while the LRH-CBF takes the desired action into account.
    \end{remark}

    \begin{table*}[ht]
        \centering
        \vspace{0.3cm}
        \begin{tabular}%{lrrrrrr}
        {lccccc}
            \toprule
            Method & Collisions & Time to Goal [s] & Path Length [m] & Control Intervention [m/s$^2$] & Min Clearance [m]\\
            \midrule
            MDH-CBF & 0 & 13.820 $\pm$ 1.681 & 22.658 $\pm$ 1.835 & 6.754 $\pm$ 1.456 & 0.239 $\pm$ 0.124 \\
            MHH-CBF & 0 & 11.134 $\pm$ 1.112 & 22.196 $\pm$ 1.010 & 2.739 $\pm$ 0.899 & 0.161 $\pm$ 0.079 \\
            C-LRH-CBF & 0 & 10.944 $\pm$ 1.066 & 22.209 $\pm$ 0.933 & 2.509 $\pm$ 0.848 & \textbf{0.048 $\pm$ 0.047} \\
            \hline
            VO-MDH-CBF & 0 & 12.490 $\pm$ 0.695 & \textbf{21.579 $\pm$ 0.177} & 5.444 $\pm$ 0.974 & 0.623 $\pm$ 0.200 \\
            VO-LRH-CBF & 0 & \textbf{10.301 $\pm$ 0.223} & 21.679 $\pm$ 0.124 & \textbf{2.084 $\pm$ 0.466} & 0.163 $\pm$ 0.089 \\
            \bottomrule
            
        \end{tabular}
        \vspace{0.05cm}
        \caption{The statistics after 500 runs of Example \ref{exp: B}. Bold marks the best performance for each metric.}
        \label{tab: stats}
    \end{table*}

    % It is noted that the adaptive H-CBF allows the agent to make a smaller detour, avoiding the retardation at around $t=3$ in Figure~\ref{fig: Exp B spatial}. 
    % To compare the performance of the two approaches, the difference $||\mathbf{u}-\mathbf{u}_{\text{des}}||$ and distance to the goal against time are presented in Figure \ref{fig: Exp B distance}. 
    
    % The adaptive H-CBF gets within 1 m of the goal in 6.1 s, compared to 9.0 s for the orthogonal. Analysing the CBF influence $||\mathbf{u}-\mathbf{u}_{\text{des}}||$, it is noted that 
    
    \begin{figure}[!t]
    \centering
    \def\svgwidth{0.45\textwidth}\footnotesize
    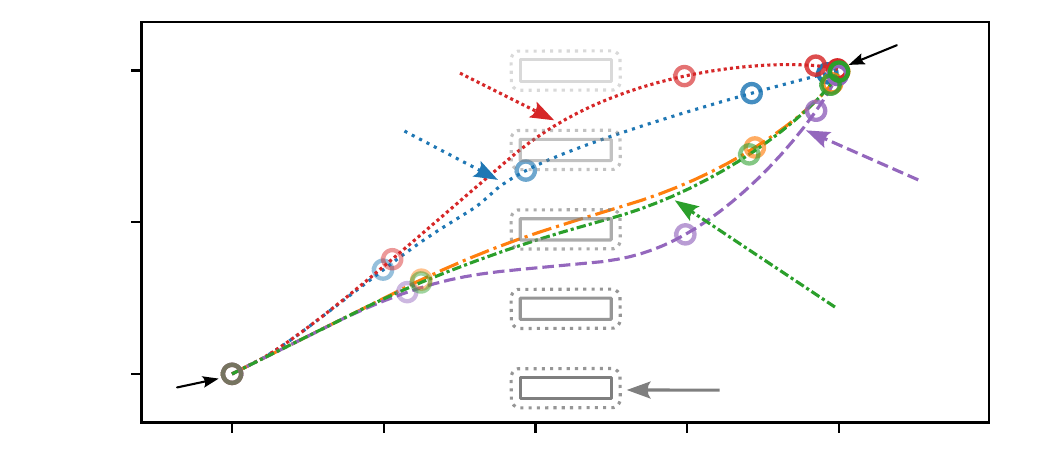
    \caption{The trajectories of one instance of Example \ref{exp: B}. 
    Concurrent snapshots of the agent and the moving obstacle are shown, with older ones increasingly opaque.
    The agent is driven towards the goal with a PD controller, using the five different types of H-CBFs. 
    Note how the C-LRH-CBF is allowed to move closer to the obstacle than the other methods, enabling it to reach the goal first.
    % This is expected, as the C-LRH-CBF is searching for a $\theta$ that minimises $||\mathbf{u}-\mathbf{u}_{\text{des}}||$.
    }
    \label{fig: Exp B spatial}
    \end{figure}
    
    \begin{figure}[!t]
    \centering
    \def\svgwidth{0.45\textwidth}\footnotesize
    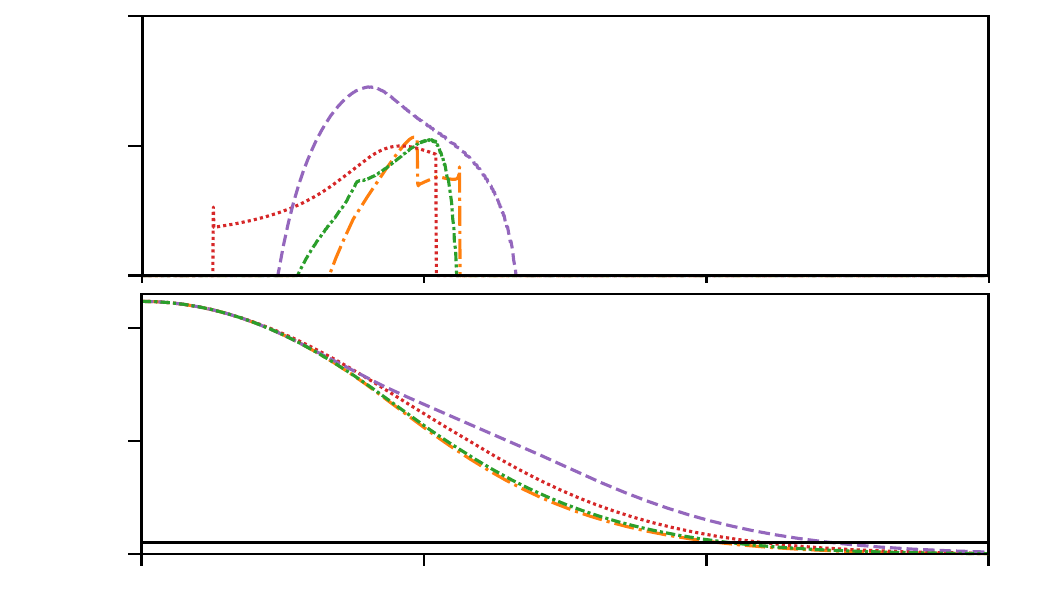
    \caption{Control deviation and distance to goal for one simulation of Example \ref{exp: B}. Upper: The difference between the desired control and the safe control $||\mathbf{u}-\mathbf{u}_\text{des}||$ over time. The C-LRH-CBF starts to intervene last and restricts less than the other methods. Lower: The distance to the goal as a function of time. The C-LRH-CBF is first, with the MHH-CBF and VO-LRH-CBF close behind.}
    
    % Distance to the goal as a function of time, for example \ref{exp: B}. The orthogonal H-CBF arrives within 1 m of the goal in 9.0 s, while the least restrictive H-CBF does so in 6.1 s, as a result of an earlier and smaller correction of the path. 
    
    \label{fig: Exp B distance}
    \end{figure}
\end{example}

\begin{example} \label{exp: C}
    In this example, the agent encounters five obstacles of different shapes and sizes, both moving and stationary.
    Figure \ref{fig: Exp C spat} shows the trajectories, and Figure \ref {fig: Exp C dist} shows the clearance to the closest obstacle and the distance to the goal. First, note that all methods ensure safety. Second, the C-LRH-CBF still gets within 1 m of the goal first, closely followed by the MHH-CBF and the VO-LRH-CBF.    

    \begin{figure}[!t]
    \centering
    \def\svgwidth{0.45\textwidth}\footnotesize
    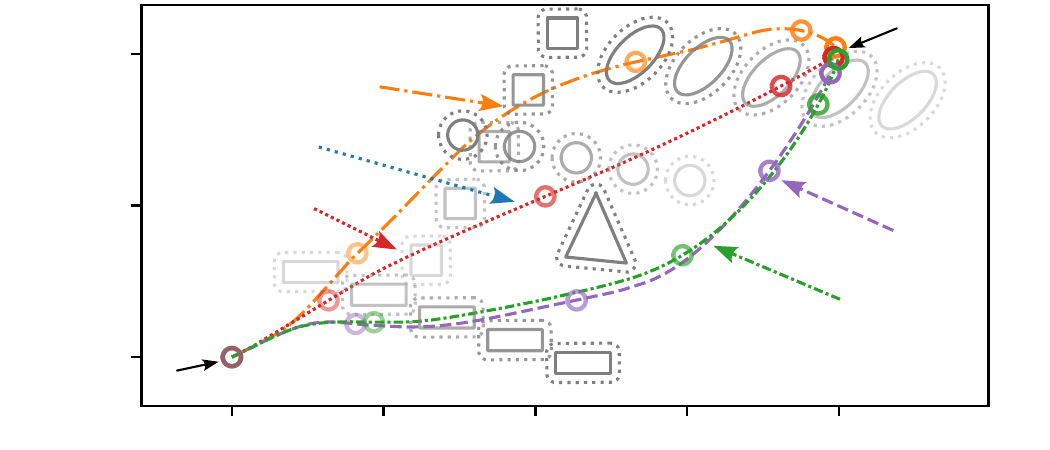
    \caption{The trajectories in Example \ref{exp: C}, featuring moving and irregularly shaped obstacles. Concurrent snapshots of the agent and the moving obstacle are shown, with older ones increasingly opaque. Notice how the C-LRH-CBF finds a longer but faster trajectory.}
    \label{fig: Exp C spat}
    \end{figure}

    \begin{figure}[!t]
    \centering
    \def\svgwidth{0.45\textwidth}\footnotesize
    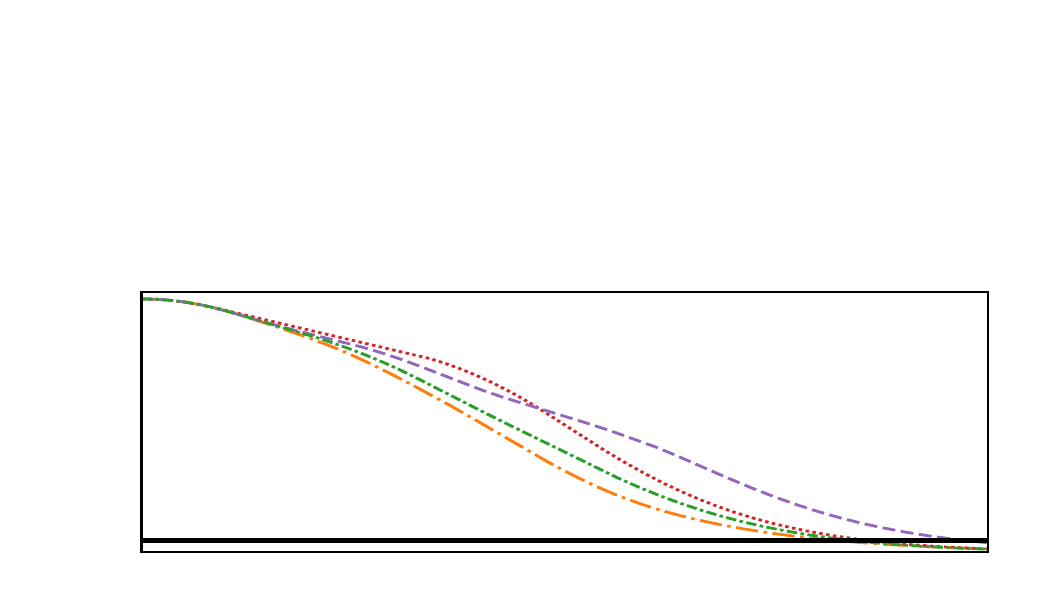
    \caption{The clearance to the closest obstacle (in configuration space) and distance to goal corresponding to Example \ref{exp: C}. Upper: note that methods ensure safety. Lower: notice how the three optimising methods perform better than the fixed hyperplane ones.}
    \label{fig: Exp C dist}
    \end{figure}
\end{example}

\section{Conclusion}
\label{section: Conclusion}
In earlier work, the CBF is chosen before it is time to find a safe control. 
In this paper, we suggest selecting the CBF with the desired control in mind.
This way, we can find the least restrictive CBF with respect to what we want to achieve.
The advantages of the proposed approach were illustrated in five examples, including moving obstacles.

\addtolength{\textheight}{0cm}   % This command serves to balance the column lengths

%%%%%%%%%%%%%%%%%%%%%%%%%%%%%%%%%%%%%%%%%%%%%%%%%%%%%%%%%%%%%%%%%%%%%%%%%%%%%%%%

\bibliographystyle{IEEEtran}
\bibliography{CBFintro,MyLibraryPetter}

\begin{thebibliography}{10}
\providecommand{\url}[1]{#1}
\csname url@samestyle\endcsname
\providecommand{\newblock}{\relax}
\providecommand{\bibinfo}[2]{#2}
\providecommand{\BIBentrySTDinterwordspacing}{\spaceskip=0pt\relax}
\providecommand{\BIBentryALTinterwordstretchfactor}{4}
\providecommand{\BIBentryALTinterwordspacing}{\spaceskip=\fontdimen2\font plus
\BIBentryALTinterwordstretchfactor\fontdimen3\font minus \fontdimen4\font\relax}
\providecommand{\BIBforeignlanguage}[2]{{%
\expandafter\ifx\csname l@#1\endcsname\relax
\typeout{** WARNING: IEEEtran.bst: No hyphenation pattern has been}%
\typeout{** loaded for the language `#1'. Using the pattern for}%
\typeout{** the default language instead.}%
\else
\language=\csname l@#1\endcsname
\fi
#2}}
\providecommand{\BIBdecl}{\relax}
\BIBdecl

\bibitem{prajna_framework_2007}
S.~Prajna, A.~Jadbabaie, and G.~J. Pappas, ``A {Framework} for {Worst}-{Case} and {Stochastic} {Safety} {Verification} {Using} {Barrier} {Certificates},'' \emph{IEEE Transactions on Automatic Control}, vol.~52, no.~8, pp. 1415--1428, Aug. 2007.

\bibitem{ames_control_2014}
A.~D. Ames, J.~W. Grizzle, and P.~Tabuada, ``Control barrier function based quadratic programs with application to adaptive cruise control,'' in \emph{53rd {IEEE} {Conference} on {Decision} and {Control}}, Dec. 2014, pp. 6271--6278.

\bibitem{ames_control_2019}
A.~D. Ames, S.~Coogan, M.~Egerstedt, G.~Notomista, K.~Sreenath, and P.~Tabuada, ``Control {Barrier} {Functions}: {Theory} and {Applications},'' in \emph{2019 18th {European} {Control} {Conference} ({ECC})}, Jun. 2019, pp. 3420--3431.

\bibitem{hsuSafetyFilterUnified2024}
K.-C. Hsu, H.~Hu, and J.~F. Fisac, ``The {{Safety Filter}}: {{A Unified View}} of {{Safety-Critical Control}} in {{Autonomous Systems}},'' \emph{Annual Review of Control, Robotics, and Autonomous Systems}, vol.~7, no.~1, pp. 47--72, Jul. 2024.

\bibitem{chen_backup_2021}
Y.~Chen, M.~Jankovic, M.~Santillo, and A.~D. Ames, ``Backup {Control} {Barrier} {Functions}: {Formulation} and {Comparative} {Study},'' Apr. 2021, arXiv:2104.11332 [eess].

\bibitem{borrmann_control_2015}
U.~Borrmann, L.~Wang, A.~D. Ames, and M.~Egerstedt, ``Control {Barrier} {Certificates} for {Safe} {Swarm} {Behavior},'' \emph{IFAC-PapersOnLine}, vol.~48, no.~27, pp. 68--73, Jan. 2015.

\bibitem{wang_safety_2017}
L.~Wang, A.~D. Ames, and M.~Egerstedt, ``Safety {Barrier} {Certificates} for {Collisions}-{Free} {Multirobot} {Systems},'' \emph{IEEE Transactions on Robotics}, vol.~33, no.~3, pp. 661--674, Jun. 2017.

\bibitem{molnar_collision_2025}
T.~G. Molnar, S.~K. Kannan, J.~Cunningham, K.~Dunlap, K.~L. Hobbs, and A.~D. Ames, ``Collision {Avoidance} and {Geofencing} for {Fixed}-{Wing} {Aircraft} {With} {Control} {Barrier} {Functions},'' \emph{IEEE Transactions on Control Systems Technology}, vol.~33, no.~5, pp. 1493--1508, Sep. 2025.

\bibitem{funada_collision_2025}
R.~Funada, K.~Nishimoto, T.~Ibuki, and M.~Sampei, ``Collision {Avoidance} for {Ellipsoidal} {Rigid} {Bodies} {With} {Control} {Barrier} {Functions} {Designed} {From} {Rotating} {Supporting} {Hyperplanes},'' \emph{IEEE Transactions on Control Systems Technology}, vol.~33, no.~1, pp. 148--164, Jan. 2025.

\bibitem{thyri_reactive_2020}
E.~H. Thyri, E.~A. Basso, M.~Breivik, K.~Y. Pettersen, R.~Skjetne, and A.~M. Lekkas, ``Reactive collision avoidance for {ASVs} based on control barrier functions,'' in \emph{2020 {IEEE} {Conference} on {Control} {Technology} and {Applications} ({CCTA})}, Aug. 2020, pp. 380--387.

\bibitem{liu_safety-critical_2025}
S.~Liu, Y.~Mao, and C.~A. Belta, ``Safety-{Critical} {Planning} and {Control} for {Dynamic} {Obstacle} {Avoidance} {Using} {Control} {Barrier} {Functions},'' in \emph{2025 {American} {Control} {Conference} ({ACC})}, Jul. 2025, pp. 348--354.

\bibitem{tooranjipour_lidar-based_2025}
P.~Tooranjipour and B.~Kiumarsi, ``{LiDAR}-based {Model} {Predictive} {Control} using {Control} {Barrier} {Functions},'' in \emph{2025 {American} {Control} {Conference} ({ACC})}, Jul. 2025, pp. 315--322.

\bibitem{thirugnanam_safety-critical_2022}
A.~Thirugnanam, J.~Zeng, and K.~Sreenath, ``Safety-{Critical} {Control} and {Planning} for {Obstacle} {Avoidance} between {Polytopes} with {Control} {Barrier} {Functions},'' in \emph{2022 {International} {Conference} on {Robotics} and {Automation} ({ICRA})}, May 2022, pp. 286--292.

\bibitem{wu_optimization-free_2025}
S.~Wu, Y.~Fang, N.~Sun, B.~Lu, X.~Liang, and Y.~Zhao, ``Optimization-{Free} {Smooth} {Control} {Barrier} {Function} for {Polygonal} {Collision} {Avoidance},'' \emph{IEEE Transactions on Cybernetics}, vol.~55, no.~9, pp. 4257--4269, Sep. 2025.

\bibitem{wei_collision_2025}
S.~Wei, R.~Khorrambakht, P.~Krishnamurthy, V.~Mariano~Gonçalves, and F.~Khorrami, ``Collision {Avoidance} for {Convex} {Primitives} via {Differentiable} {Optimization}-{Based} {High}-{Order} {Control} {Barrier} {Functions},'' \emph{IEEE Transactions on Control Systems Technology}, pp. 1--16, 2025.

\bibitem{notomista_reactive_2025}
G.~Notomista, G.~P.~T. Choi, and M.~Saveriano, ``Reactive {Robot} {Navigation} {Using} {Quasi}-{Conformal} {Mappings} and {Control} {Barrier} {Functions},'' \emph{IEEE Transactions on Control Systems Technology}, vol.~33, no.~3, pp. 928--939, May 2025.

\bibitem{dai_safe_2023}
B.~Dai, R.~Khorrambakht, P.~Krishnamurthy, V.~Gonçalves, A.~Tzes, and F.~Khorrami, ``Safe {Navigation} and {Obstacle} {Avoidance} {Using} {Differentiable} {Optimization} {Based} {Control} {Barrier} {Functions},'' \emph{IEEE Robotics and Automation Letters}, vol.~8, no.~9, pp. 5376--5383, Sep. 2023.

\bibitem{huang_dynamic_2025}
J.~Huang, J.~Zeng, X.~Chi, K.~Sreenath, Z.~Liu, and H.~Su, ``Dynamic {Collision} {Avoidance} {Using} {Velocity} {Obstacle}-{Based} {Control} {Barrier} {Functions},'' \emph{IEEE Transactions on Control Systems Technology}, vol.~33, no.~5, pp. 1601--1615, Sep. 2025.

\bibitem{roncero_multi-agent_2025-1}
A.~S. Roncero, R.~I.~C. Muchacho, and P.~Ögren, ``Multi-{Agent} {Obstacle} {Avoidance} {Using} {Velocity} {Obstacles} and {Control} {Barrier} {Functions},'' in \emph{2025 {IEEE} {International} {Conference} on {Robotics} and {Automation} ({ICRA})}, May 2025, pp. 6638--6644.

\bibitem{ghaffariSafetyVerificationUsing2018}
A.~Ghaffari, I.~Abel, D.~Ricketts, S.~Lerner, and M.~Krstic, ``Safety {{Verification Using Barrier Certificates}} with {{Application}} to {{Double Integrator}} with {{Input Saturation}} and {{Zero-Order Hold}},'' in \emph{2018 {{Annual American Control Conference}} ({{ACC}})}.\hskip 1em plus 0.5em minus 0.4em\relax Milwaukee, WI: IEEE, Jun. 2018, pp. 4664--4669.

\bibitem{agrawal_discrete_2017}
A.~Agrawal and K.~Sreenath, ``\BIBforeignlanguage{en}{Discrete {Control} {Barrier} {Functions} for {Safety}-{Critical} {Control} of {Discrete} {Systems} with {Application} to {Bipedal} {Robot} {Navigation}},'' in \emph{\BIBforeignlanguage{en}{Robotics: {Science} and {Systems} {XIII}}}, Jul. 2017.

\bibitem{rockafellarConvexAnalysis1970}
T.~Rockafellar, \emph{Convex {{Analysis}}}.\hskip 1em plus 0.5em minus 0.4em\relax Princeton University Press, 1970.

\bibitem{lavallePlanningAlgorithms2006}
S.~M. LaValle, \emph{Planning {{Algorithms}}}, 1st~ed.\hskip 1em plus 0.5em minus 0.4em\relax Cambridge University Press, May 2006.

\bibitem{andersson_casadi_2019}
J.~A.~E. Andersson, J.~Gillis, G.~Horn, J.~B. Rawlings, and M.~Diehl, ``\BIBforeignlanguage{en}{{CasADi}: a software framework for nonlinear optimization and optimal control},'' \emph{\BIBforeignlanguage{en}{Mathematical Programming Computation}}, vol.~11, no.~1, pp. 1--36, Mar. 2019.

\bibitem{breeden_compositions_2023}
J.~Breeden and D.~Panagou, ``Compositions of {Multiple} {Control} {Barrier} {Functions} {Under} {Input} {Constraints},'' in \emph{2023 {American} {Control} {Conference} ({ACC})}, May 2023, pp. 3688--3695.

\end{thebibliography}

% \begin{thebibliography}{99}

% \end{thebibliography}

\end{document}